\documentclass[sigconf, nonacm]{acmart}

\usepackage{float}
\usepackage[normalem]{ulem}
\usepackage{makecell,multirow,diagbox}  
\usepackage{float}
\usepackage{adjustbox}
\usepackage{lscape}
\usepackage{amsfonts,amssymb}
\usepackage{amsthm}
\usepackage{amsmath}
\usepackage{amssymb}
\usepackage{color}
\usepackage{xcolor}
\usepackage{algorithm}
\usepackage[noend]{algorithmic}
\usepackage{subfigure}
\usepackage{enumitem}
\usepackage{color, colortbl}
\usepackage{booktabs}
\usepackage{graphicx}
\usepackage{url}
\useunder{\uline}{\ul}{}

\definecolor{gray}{HTML}{545454}

\newcommand{\DG}{\Delta G}
\usepackage{amsmath}
\usepackage{amsfonts}

\definecolor{gray}{HTML}{545454}

\newcommand{\real}{\mathbb R}

\newcommand{\mcF}{\mathcal{F}}

\newcommand{\mcL}{\mathcal{L}}

\newcommand{\mcS}{\mathcal{S}}

\newcommand{\mbp}{\mathbf{p}}

\newcommand{\ie}{\emph{i.e., }}

\newcommand{\aka}{\emph{aka. }}
\newcommand{\vs}{\emph{v.s. }}

\newtheorem{theorem}{Theorem}[section]
\newtheorem{definition}{Definition}[section]
\newtheorem{assumption}{Assumption}[section]

\newtheorem{lemma}{Lemma}[section]
\newtheorem{proposition}{Proposition}[section]

\newtheorem*{theorem*}{Theorem}
\newtheorem*{definition*}{Definition}
\newtheorem*{assumption*}{Assumption}
\newtheorem*{conjecture*}{Conjecture}
\newtheorem*{claim*}{Claim}
\newtheorem*{lemma*}{Lemma}
\newtheorem*{proposition*}{Proposition}
\newtheorem*{property*}{Property}
\newtheorem*{fact*}{Fact}
\newtheorem*{corollary*}{Corollary}
\newtheorem*{example*}{Example}
\newtheorem*{remark*}{Remark}
\newtheorem*{exercise*}{Exercise}

\begin{document}
\title{A Bargaining-based Approach for Feature Trading in Vertical Federated Learning}

\author{Yue Cui}
\affiliation{%
  \institution{The Hong Kong University \\of Science and Technology}
  \city{Hong Kong SAR}
  \country{China}
}
\email{ycuias@cse.ust.hk}

\author{Liuyi Yao}
\affiliation{%
  \institution{Damo Academy\\  Alibaba Group}
  \city{Hangzhou}
  \country{China}
}
\email{yly287738@alibaba-inc.com}


\author{Zitao Li}
\affiliation{%
  \institution{Damo Academy\\  Alibaba Group}
  \city{Seattle}
  \country{USA}
}
\email{zitao.l@alibaba-inc.com}

\author{Yaliang Li}
\affiliation{%
  \institution{Damo Academy\\  Alibaba Group}
  \city{Seattle}
  \country{USA}
}
\email{yaliang.li@alibaba-inc.com}

\author{Bolin Ding}
\affiliation{%
  \institution{Damo Academy\\ Alibaba Group}
  \city{Seattle}
  \country{USA}
}
\email{bolin.ding@alibaba-inc.com}

\author{Xiaofang Zhou}
\affiliation{%
  \institution{The Hong Kong University \\of Science and Technology}
  \city{Hong Kong SAR}
  \country{China}
}
\email{zxf@cse.ust.hk}

\begin{abstract}
Vertical Federated Learning (VFL) has emerged as a popular machine learning paradigm, enabling model training across the data and the task parties with different features about the same user set while preserving data privacy. In production environment, VFL usually involves one task party and one data party. Fair and economically efficient feature trading is crucial to the commercialization of VFL, where the task party is considered as the data consumer who buys the data party's features. However, current VFL feature trading practices often price the data party's data as a whole and assume transactions occur prior to the performing VFL. Neglecting the performance gains resulting from traded features may lead to underpayment and overpayment issues. In this study, we propose a bargaining-based feature trading approach in VFL to encourage economically efficient transactions. Our model incorporates performance gain-based pricing, taking into account the revenue-based optimization objectives of both parties. We analyze the proposed bargaining model under perfect and imperfect performance information settings, proving the existence of an equilibrium that optimizes the parties' objectives. Moreover, we develop performance gain estimation-based bargaining strategies for imperfect performance information scenarios and discuss potential security issues and solutions. Experiments on three real-world datasets demonstrate the effectiveness of the proposed bargaining model.
\end{abstract}
\maketitle



\section{Introduction}

Federated learning (FL) has become a popular machine learning paradigm for that it allows model training to happen across multiple devices or institutions, \aka clients, while keeping their data localized. 
Vertical Federated Learning (VFL) is a subcategory of FL that enables the training of clients' datasets sharing the data from the same set of users while holding different features. 

The increasing demand for VFL is evident in recent industrial trends~\cite{wu2020privacy,cheng2021secureboost,chen2021secureboost+,niknam2020federated}. Organizations with limited or fragmented datasets are driven to partner with complementary data sources to amplify the training of machine learning models. Alongside, global concerns about data breaches and privacy violations have amplified the urgency for enhanced security measures. Consequently, a rise in privacy-focused VFL platforms and initiatives has been observed \cite{liu2021fate,xie2023federatedscope,he2020fedml}. While multi-party VFL is a subject of extensive academic research~\cite{liu2022vertical}, insight from the China Academy of Information and Communications Technology's (CAICT) White Paper on Federated Learning Scenario Applications\footnote{\url{http://www.caict.ac.cn/kxyj/qwfb/ztbg/202202/P020220222528294962585.pdf} \label{footnote:white}}, \cite{xie2023federatedscope}, and \cite{liu2021fate}, indicates that in practical applications or production environments, VFL is often implemented with a single task party and data party. Specifically, according to the white paper, this 1v1 VFL paradigm has been employed by commercial banks to amalgamate external data when constructing joint anti-fraud models, establish potential insurance users alongside insurance companies, as well as by advertisers to conduct user modeling with data of external media platforms. The commercial relevance and economic impact attributed to this 1v1 VFL model is steadily growing. Therefore, in this paper, we focus on the VFL setting characterized by a singular task party and a singular data party.

In the current era, data has transformed into a valuable commodity, often considered as goods with specific prices \cite{pei2021data,pei2020survey,jia2023research,chai2022selective,acemoglu2022too}. This valuation is prominently seen in digital products, where information serves as both a resource and product, forming the backbone of modern services and technologies. The two key roles in VFL can also be considered as data consumer and the data provider: the ``task party", responsible for labels and performing downstream tasks, and the ``data party", supplying the essential feature data. Current works in VFL mostly concentrate on protecting the security and privacy of data or information transmission between the task party and the data party, ensuring that collaboration occurs without compromising the integrity and confidentiality of the involved parties. However, the intricate aspect of data valuation, which underpins many economic decisions in other contexts, has not been extensively explored or studied. 

In the current production environment, data trading in VFL is mostly conducted at party level. More specifically, the data party names a price for its feature, and when a task party is paired with it for VFL, the task party buys all features offered by the data party. However, this may lead to undesired outcomes: 1) for the task party, buying all data from the data party, it also pays for features that may not be useful in enhancing model performance, and 2) for the data party, as the utility of the data is contingent upon task parties' different application scenarios and production environments, it may under-estimate/over-estimate the value of the features. This indiscriminate and one-shot trading of all features between parties highlights the need for a more nuanced and economically efficient approach for feature trading in VFL. Bargaining has proven to be effective in various economic contexts where there are two players \cite{nash1950bargaining} participated, which permits parties to negotiate and agree upon the value of individual features or combinations of features, resulting in a more flexible and efficient allocation of resources. The iterative process where a party could revisit its offers also encourages the achieving of mutual benefit. Inspired by this, we here propose to introduce bargaining for the feature trading in VFL.

Bargaining in the context of VFL diverges from traditional data asset markets in several key aspects, and this brings non-trivial challenges. First, in conventional one-to-one bargaining in data markets, the buyer seeks to obtain specific goods, such as datasets, from the seller \cite{jia2023research}. The participants then negotiate the price of these goods. In contrast, within the setting of VFL, bargaining is result-oriented. For the task party, it targets to maximize the net profit of using the bought feature to perform VFL. For the data party, it aims at pricing the features with their utility considered and maximizing the payment of selling features. As a consequence, the pricing mechanism should take model performance gain into consideration. 

Second, in conventional one-one bargaining, the information advantage is clearer, \ie the seller has more information about the value of goods, and, thus is in a better place to take the lead during the game \cite{nash1950bargaining}. While in VFL market, the value of features is hard to be known in advance to both the data party (seller) and the task party (buyer). This uncertainty arises primarily for two reasons: 1) a newly paired task party and data party have no information about the performance gain that their collaboration could result in until VFL is conducted; and 2) even if the performance gain is known in advance, the data party lacks insights about the task party's model utility, which is related to the production environment and not revealed since they are proprietary business information belonging to the task party. Therefore, it is important to manage the information on the market and design fair bargaining model accordingly.

To deal with these problems, we first define a performance gain related payment function to price the outcome of a VFL course, based on which we formulate the objectives of the participants in the form of revenues. A bargaining model with the key idea that the task party targets a certain performance gain, offering a quoted price to the data party and the data party responds with products, \ie feature bundles, is proposed. The negotiation process continues until both parties agree on a feature bundle - payment matching (acceptance) or until they are unable to reach an agreement (breakdown). The task party is designated to take the lead during the game because it is more sensitive to the outcome of the VFL, as well as be aware of the budget and utility rate of its model. 

We analyze the bargaining under perfect performance information and imperfect performance information settings, distinguished by if the performance gain resulting from a feature bundle can be known in advance. The existence of an equilibrium under perfect performance information is proved theoretically. As bargaining is an iterative process with time cost, we further discuss the equilibrium by considering the cost factor. Further, bargaining strategies are established in imperfect performance information setting based on learning-based estimation. 

Key contributions of this paper are summarized as follows.
\begin{itemize}
	\item We identify the need for an economically efficient approach for feature trading in VFL and propose to introduce bargaining in VFL market.
	\item A bargaining model is designed with performance gain based pricing. We analyze the proposed bargaining model under perfect performance information setting with and without bargaining cost considered, proving the existence of an equilibrium that optimizes the objectives of the task party and the data party.
	\item Performance gain estimation based bargaining strategies are proposed under the imperfect performance information setting for both parties. We further discuss possible security issues and solutions for the proposed VFL market.
	\item We conduct extensive experiments on three real-world datasets, verifying the effectiveness of the proposed bargaining model.
\end{itemize}

\section{Preliminary}
\label{sec:pre}

\begin{table}[t]
\centering
\vspace{-0.1cm}
\caption{Notations and descriptions.}
\vspace{-0.2cm}
\begin{tabular}{@{}ll@{}}
\specialrule{0.3mm}{0em}{0em} 
Notations           & Descriptions                                                                                                                                                                                          \\ \hline
$T$                 & The $T$-th bargaining round                                                                                                                                                                           \\ \hline
$F$,$\mcF$          & Feature bundle, a set of feature bundles                                                                                                                                                              \\ \hline
$\DG$               & Performance gain of the task party                                                                                                                                                                    \\ \hline
$\mbp=(p,P_0,P_h)$  & \begin{tabular}[c]{@{}l@{}}Task party's quoted price, composed of payment \\rate $p$, base payment $P_0$, and highest payment $P_h$\end{tabular}                                                  \\ \hline
$(p_{l,i},P_{l,i})$ & \begin{tabular}[c]{@{}l@{}}Data party's reserved price of feature bundle $F_i$, \\ composed of reserved price of payment rate $p_{l,i}$, \\ and reserved price of base payment $P_{l,i}$\end{tabular} \\ \hline
$u$                 & Task party's utility rate                                                                                                                                                                             \\ \hline
$R_d$               & Revenue received by the data party                                                                                                                                                                    \\ \hline
$R_t$               & Revenue received by the task party                                                                                                                                                                    \\ \hline
$f$, $\theta_f$     & \begin{tabular}[c]{@{}l@{}}Task party's performance gain estimation model, \\ model weights\end{tabular}                                                                                              \\ \hline
$g$, $\theta_g$     & \begin{tabular}[c]{@{}l@{}}Data party's performance gain estimation model, \\ model weights\end{tabular}                                                                                              \\ \specialrule{0.3mm}{0em}{0em} 
\end{tabular}
\label{tb:notation}
\vspace{-0.1cm}
\end{table}

Key notations and descriptions used in this paper are summarized in Table \ref{tb:notation}. We assume that both parties have $n$ aligned training samples. The task party locally maintains a dataset $\{X_t, Y\}$ of the $n$ samples, where $X \in \real^{n\times d_t}$ are $d_t$ features and $Y \in \real^n$ are the labels. The data party locally maintains a dataset $\{X_d\}$, where $X_d \in \real^{n\times d_d}$ are $d_d$ features of the n samples. 

We denote the performance of the model trained by the task party itself as $M_0$ and the performance of a model trained after VFL as $M$. The performance gain $\DG$ is calculated as the relative improvement of performance, \ie 
\begin{equation}
\DG=\frac{M-M_0}{M_0}.
\label{eq:deltaG}
\end{equation}

In the above formula, we assume the performance metric the higher the better. If not, the sign of $\DG$ needs to be reversed. 

\begin{definition}{Feature Bundle.}
Given the $d_t$ features of the data party, denoted as $\mcS$, a feature bundle denoted as $F$ is a combination of individual features in $\mcS$, \ie $F \subseteq \mcS$. Denote the set of feature bundles as $\mcF$.
\end{definition}
Feature bundles are the goods selling on the VFL market.

\begin{definition}{Quoted Price.}
The task party quotes for a feature bundle of the data party by stating the payment rate $p$, the base payment $P_0$, and the highest payment $P_h=P_0+C$, where $C\geq 0$. $\textbf{p}=(p, P_0, P_h)$ is referred to as the quoted price.
\end{definition}

\begin{definition}{Payment.}
Given the performance gain $\Delta G$ in a VFL course, the payment received by the data party is calculated as:
\begin{equation}
\min\{\max\{P_0,P_0+p\Delta G\},P_h\}.
\label{eq:payment}
\end{equation}
\end{definition}
The payment function indicates that the data party can receive a payment of at least $P_0$, at most $P_h$, and somewhere in between depending on $p$ and $\DG$.

\begin{definition}{Reserved Price.}
Given a feature bundle $F_i$, the reserved price, denoted as $(p_{l,i},P_{l,i})$, is a value privately maintained by the data party, which represents the pre-defined minimum base payment and minimum payment rate it expects to receive if offering the feature bundle to the task party. 
\end{definition}

Note that the reserved price can be considered as cost-related. For example, a feature bundle of a larger number of features may have higher reserved price as the collecting cost of the feature bundle is higher than a feature bundle of a smaller number of features. 

\section{The Bargaining-based VFL Market}

\subsection{Assumptions}
Before going deep into the proposed method, we first illustrate several essential assumptions of the VFL market.

\begin{assumption}[Individual Rationality]
The data party and the task party are rational, seeking to maximize their revenues, respectively.
\end{assumption}

\begin{assumption}[Task Party Leads]
The task party takes the lead during the game by offering the quoted price, and the data party reacts by offering the feature bundle. 
\end{assumption}

\begin{assumption}[Benign Clients]
\label{ass:beging}
The features offered by the data party and the performance gain reported by the task party are not manipulated. 
\end{assumption}

The intention of using Assumption \ref{ass:beging} is to exclude potential adversary participants that attack the FL environment with malicious behaviors, such as submitting noise instead of real model performance, so as to focus on the design of the bargaining model.

\subsection{Objectives of Participants}

Given the payment function defined in Section \ref{sec:pre}, the objectives of the task party and the data party can be formulated in the form of revenues. Typically, a higher-performing model generates greater utility for its owner. The task party enters the VFL market as a featured buyer, seeking to improve its model's performance with minimal feature buying costs. Let $u$ denote the utility increase resulting from a unit of performance gain. The task party's objective is to determine the best quoted price that maximizes net profit. Denote the net profit it can receive as $R_t$, the objective of the task party can be mathematically expressed as:

\begin{equation}
R_t=\max_{p,P_0,P_h} u \Delta G -\min\{\max\{P_0,P_0+p\Delta G\},P_h\}.
\label{eq:obj_task}
\end{equation}

The data party, on the other hand, aims to maximize the monetary payment it receives. However, the highest payment the data party can receive is limited by $P_h$ of the quoted price, which means that an overqualified feature bundle that achieves performance gain greater than $(P_h-P_0)/p$ will not be fairly paid. Therefore, a reasonable strategy for the data party is to offer a feature bundle that achieves performance gain as close to $(P_h-P_0)/p$ as possible. Mathematically, denote the payment it receives as $R_d$, the objective of the data party can be formulated as:

\begin{equation}
R_d=\min_{F} |P_h-\max\{P_0,P_0+p\Delta G\}|.
\label{eq:obj_data}
\end{equation}

\subsection{Iterative Bargaining}
\label{sec:workflow}


We design an interactive process where the feature trading operates in multiple rounds of price-feature offering until certain termination conditions are met. For each full bargaining round, the workflow can be described as follows,

\begin{itemize}
\item Step 1. The task party initiates the process by announcing $(p, P_0, P_h)$, which serves as a quote for features in the upcoming VFL course.
\item Step 2. The data party responds by determining a feature bundle $F$ to offer to the task party.
\item Step 3. The two parties proceed with VFL course on $F$.

\end{itemize}
In Step 1 and Step 2, there are termination conditions applied regarding the received $\DG$ in VFL course and their objectives.

\begin{figure}
\centering

\subfigure[The payment received by the data party as a function of $\DG$.]{
\centering
\includegraphics[width=0.5\linewidth]{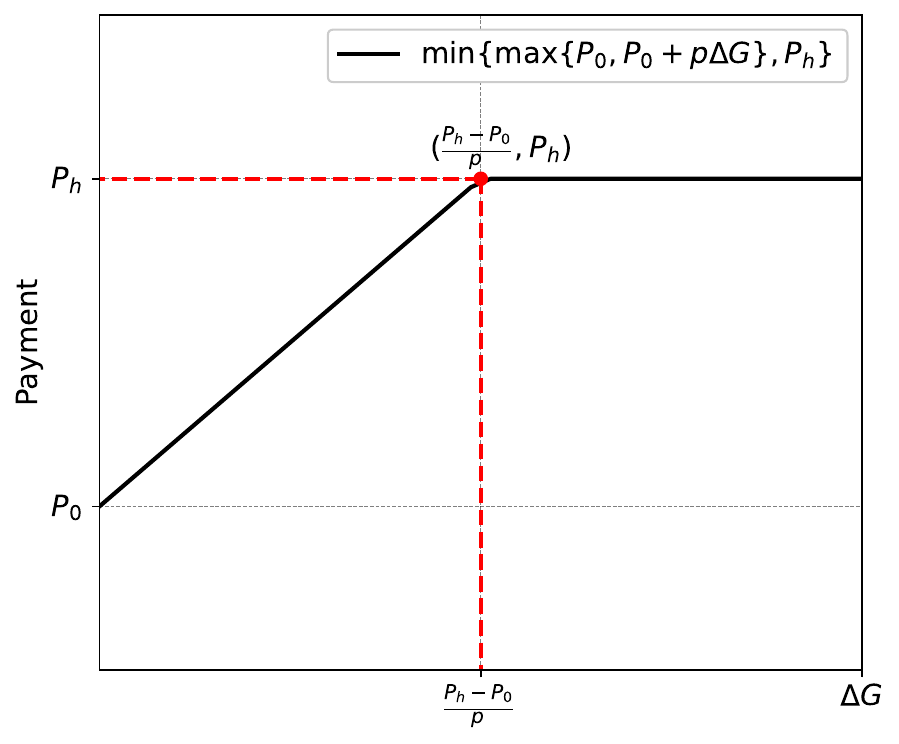}
\label{fig:data_payment}
}%
\subfigure[The net profit of the task party as a function of $\DG$.]{
\centering
\includegraphics[width=0.5\linewidth]{./figs/payment_function}
\label{fig:task_netprofit}
}%
\caption{$\DG$ \vs the objective functions of the data party and the task party.}
\label{fig:function}
\vspace{-0.4cm}
\end{figure}
We next discuss the bargaining strategies, as well as termination conditions, of two parties in two settings: perfect performance information and imperfect performance information. We first analyze bargaining in the perfect information setting, where the the parties have more awareness of each other, and then extend the analysis to imperfect setting. In the perfect information setting, we consider that the two parties know in advance about the performance gain of conducting VFL on each feature bundle of the task party. And for imperfect information, we suppose the performance gain is not known in advance unless VFL is formed after a round of iteration. We assume that in both situations, the objective functions of the parties are known to each other.

\subsection{Perfect Performance Information}

For the data party, the perfect information setting refers to the fact that it is aware of the performance $\DG_i$ of any feature bundle $F_i$ it owns when paired with a data party for bargaining. In the context of the task party, the concept of perfect information entails its awareness of the performance gain associated with the feature bundles. To be more precise, it possesses knowledge of $|\mcF|$ values of $\DG$, where $\mcF$ is the set of all feature bundles owned by the data party. Achieving this state of perfect information can be facilitated through the involvement of a trustworthy third party, such as a trading platform, which can conduct pre-bargaining training for both parties. It is important to clarify that this awareness does not extend to obtaining detailed information about the specific characteristics of individual feature bundles, as this would violate the privacy-preserving principles of FL. 

\subsubsection{Perfect Performance Information: Bargaining Analysis On The Data Party}
\label{sec:per_data}
We now delve into the bargaining strategy employed by the data party in the perfect performance information setting. Given a quoted price $(p,P_0,P_h)$ presented by the task party, the data party's objective is to choose the most optimal bundle within the confines of this price. The strategic approach of the data party is as follows. Initially, the data party compares the reserved prices of feature bundles with $(p,P_0)$, and discards bundles whose reserved prices exceed this threshold. It then finds a feature bundle whose performance gain lies nearest to, but does not surpass, $(P_h-P_0)/p$. This ensures the generation of a performance gain that would yield the maximum permissible payment within the range of $(p,P_0,P_h)$, eliminating undervaluation concerns.

It's important to highlight that, under the objective defined in Equation~\ref{eq:obj_data}, the data party lacks the motivation to offer a suboptimal feature. This is rooted in the fact that its compensation is contingent on the performance gain realized by the task party. Furthermore, given that the task party takes the lead during bargaining, the reserved price plays an important role in protecting the interest of the data party. This is because the payment the data part could receive is a function of $\DG$ and upper bounded by $P_h$. As depicted in Figure \ref{fig:data_payment}, with a predetermined quoted price, the payment received by the data party is monotonically increasing to $\DG$ when $\DG \leq (P_h-P_0)/p$ and holds as $P_h$ when $\DG > (P_h-P_0)/p$. 

In the meanwhile, according to the objective defined in Equation \ref{eq:obj_task}, the task party is driven to maximize its net profit. To this end, the task party seeks to depress both $p$ and $P_0$ while providing a sufficiently enough $P_h$. This ensures that the data party, enticed by the proposed compensation, delivers the highest quality feature bundle, all while the task party incurs minimal expenses. The introduction of the reserved price acts as a countermeasure to this behavior. It guarantees that only those feature bundles conforming to the quoted price are considered in a given bargaining round.

\subsubsection{Perfect Performance Information: Bargaining Analysis On The Task Party}

As defined in Equation \ref{eq:obj_task}, the primary goal of the task party revolves around maximizing the net profit derived from a VFL course. We assume that the utility rate is larger than the payment rate due to self rationality of the task party, \ie $u>p$. As illustrated in Figure \ref{fig:task_netprofit}, the net profit demonstrates a monotonic increase with respect to $\DG$, yet it remains negative when the performance gain falls below $P_0/(u-p)$. To establish the pricing strategy adopted by the task party, we propose the following theorem.

Let $\DG_{max}$ denote the highest achievable performance gain attainable through the feature bundles. Consider a negotiation round in which the task party has presented a specific quoted price $(p, P_0, P_h)$, and subsequently, a feature bundle is selected, yielding a performance gain of $\DG$ after undergoing VFL. Denote the payment received by the data party as $R_d$, where 
\begin{equation*}
R_d=P_0+p\DG,
\end{equation*}
and the net profit received by the task party as $R_t$, where
\begin{equation*}
R_t=u\DG-P_0-p\DG.
\end{equation*}
We can formulate Theorem \ref{thm:1}.
\begin{theorem}
\label{thm:1}
There exists a quoted price $(p^*,P_0^*,P_h^*)$ that leads to the same bargaining result as $(p,P_0,P_h)$, \ie the same offered feature bundle, the same performance gain, the same net profit for the task party and the same payment on the data party, while it satisfies $(P^*_h-P^*_0)/p^*=\DG$.
\end{theorem}

\begin{proof}
We consider two situations regarding when $(P_h-P_0)/p \geq \DG_{max}$ and $(P_h-P_0)/p < \DG_{max}$.

1) Consider the scenario where $(P_h-P_0)/p \geq \DG_{max}$. According to the data party's offer strategy, the task party obtains a feature bundle that yields a performance gain of $\DG=\DG_{max}$. Consequently, the payment accrued by the data party becomes 
\begin{equation*}
R_d=P_0+p\DG_{max},
\end{equation*}
while the task party's net profit is given by
\begin{equation*}
R_t=u\DG_{max}-P_0-p\DG_{max}.
\end{equation*}

For this performance gain $\DG_{max}$ to be offered, the quoted price must satisfy the following criteria 1) $p\geq p_{l,max}$ and $P_0\geq P_{l,max}$, where $(p_{l,max},P_{l,max})$ is the reserved price of the feature bundle that generates $\DG_{max}$, 2) and $(P_h-P_0)/p \geq\DG_{max} $. Given these conditions, there exists a quoted price $(p^*,P^*_0,P_h^*)$, where
\begin{equation*}
p^*=p,P^*_0=P_0,P_h^*=p\DG_{max}+P_0\leq P_h,
\end{equation*}
based on which the data party will provide the same feature bundle with the associated performance gain $\DG_{max}$. The payment to the data party remains $R_d$, \ie
\begin{equation*}
R^*_d=P^*_0+p^*\DG_{max}=P_0+p\DG_{max}=R_d,
\end{equation*}
and the net profit of the data party continues to be $R_t$, \ie
\begin{equation*}
\begin{aligned}
R^*_t&=u\DG_{max}-(P^*_0+p^*\DG_{max})\\
&=u\DG_{max}-(P_0+p\DG_{max})=R_t.
\end{aligned}
\end{equation*}

2) Consider the scenario where $(P_h-P_0)/p < \DG_{max}$. Similarly, according to the data party's offer strategy, the task party obtains a feature bundle that yields a performance gain of $\DG=\DG_{i}$. Consequently, the payment accrued by the data party becomes $P_0+p\DG_{i}$, while the task party's net profit is given by $u\DG_{i}-P_0-p*\DG_{i}$. For this performance gain $\DG_{i}$ to be offered, the quoted price must satisfy the following criteria 1) $p\geq p_{l,i}$ and $P_0\geq P_{l,i}$, where $(p_{l,iu},P_{l,i})$ is the reserved price of the feature bundle that generates $\DG_{i}$, 2) and $(P_h-P_0)/p-\DG_{i} \leq \epsilon_1$. Given these conditions, there exists a quoted price $(p^*, P^*_0, P^*_h)$, where $p^*=p,P^*_0=P_0,P_h^*=p\DG_{i}+P_0\leq P_h$, based on which the data party will provide the same feature bundle with the associated performance gain $\DG_{i}$. The payment to the data party remains $R_d$, \ie $R^*_d=P^*_0+p^*\DG_{i}=P_0+p\DG_{i}=R_d$ and the net profit of the data party continues to be $R^*_t=u\DG_{i}-(P^*_0+p^*\DG_{i})=u\DG_{i}-P_0-p\DG_{i}=R_t$.

\end{proof}

\begin{lemma}
Under the condition of perfect information, presenting a quoted price $(p^*, P_0^*, P_h^*)$ such that $(P_h^* - P_0^*)/p^* = \Delta G$ is a weakly dominant strategy for the task party to attain a performance gain of $\Delta G$, where $\Delta G$ represents the performance improvement obtained by the task party in the current round.
\label{le:1}
\end{lemma}

\begin{proof}
Let us assume that in order to achieve the current performance gain $\Delta G$, there exists a set of quoted prices that the task party can offer, denoted as $\{(p^1,P^1_0,P^1_h),(p^2,P^2_0,P^2_h),...,(p^k,P^k_0,P^k_h)\}$. Denote the corresponding net profits as $\{R^1_t,R^2_t,...,R^k_t\}$. Among these, we identify the quoted price that yields the highest net profit as 
\begin{equation*}
(p,P_0,P_h)=\max_{(p^i,P^i_0,P^i_h),i\in{K}}R^i_t.
\end{equation*}
According to Theorem~\ref{thm:1}, there exists a quoted price $(p^*,P^*_0,P^*_h)$, where $p^*=p$, $P_0^*=P_0$, and $P^*_h=P_0+p\DG$ that satisfies $(P_h^* - P_0^*)/p^*= \Delta G$, while also yielding net profit $R^i_t$. This implies that $(p^*,P^*_0,P^*_h)$ weakly dominates other quoted prices within the set $\{(p^1,P^1_0,P^1_h)$, $(p^2,P^2_0,P^2_h)$, $\cdots$, $(p^k,P^k_0,P^k_h)\}$.
\end{proof}

Utilizing Theorem~\ref{thm:1} and Lemma~\ref{le:1}, we can establish the existence of an equilibrium price in conditions of perfect performance information, which adheres to the criterion
\begin{equation}
\frac{P_h-P_0}{p}=\DG. 
\label{eq:cri}
\end{equation}

This price demonstrates weak dominance over the income of both the task and data party. Therefore, a reasonable strategy for the task party is that it initially targets a specific performance gain value, and the bargaining process commences with an initial quoted price $(p, P_0, P_h)$ that satisfies Equation~\ref{eq:cri}. Should the transaction proceed without failure but yield a feature bundle with a performance gain below $\DG$, the task party incrementally adjusts the price, as it indicates that the target feature bundle's reserved price is higher than the current $(p,P_0)$. Another quoted price should be offered, ensuring compliance with the prescribed constraints while maximizing the net profit of the task party among the rest of the candidate price offers. The process repeats until a feature offer is presented that best matches the price offer, \ie reaching the condition where Equation \ref{eq:cri} is satisfied. Given the task party leads the bargaining, the payment received by the data party is bound by both $P_h$ and the task party's target performance gain, $\DG$, but it can also be maximized under such a quoted price, for which we refer to as the equilibrium price, as it leads to an outcome that both parties accept the bargaining.

\subsubsection{Perfect Performance Information: Termination Conditions}
\label{sec:per_termi}
Based on the above analysis, we design the following termination conditions to facilitate the reaching of the equilibrium price. Denote the quoted price offered by the task party as $(p,P_0,P_h)$, the feature bundle selected by the data party as described in \ref{sec:per_data} as $F_i$, the corresponding performance gain as $\DG_i$, and after a round of bargaining, the performance gain produces by the task party as $\DG$. The bargaining can be divided into the following cases.

On the side of the data party, during Step 2,
\begin{itemize}
\item Case 1: If there is no feature bundle $F_i$ in $\mcF$ that satisfies $p_{l,i}\leq p$ and $P_{l,i}\leq P_0$, the bargaining terminates by the data party with the transaction fails.

\item Case 2: If the selected $F_i$ satisfies $p_{l,i}\leq p$ and $P_{l,i}\leq P_0$, and $(P_h-P_0)/p-\DG_i\leq \epsilon_d$, the bargaining terminates by the data party with transaction successes and $F_i$ offered. 

\item Case 3: If the selected $F_i$ satisfies $p_{l,i}\leq p$ and $P_{l,i}\leq P_0$, while $(P_h-P_0)/p-\DG_i> \epsilon_d$, the bargaining proceeds with $F_i$ offered.
\end{itemize}

On the side of the task party, during Step 1,
\begin{itemize}

\item Case 4: If the received $\DG$ satisfies $\DG < P_0/(u-p)$, the bargaining terminates by the task party with transaction fails. 

\item Case 5: If the received $\DG$ satisfies $\DG \geq (P_h-P_0)/p-\epsilon_t$, the bargaining terminates by the task party with transaction successes and the task party makes payments. 

\item Case 6: If the received $\DG$ satisfies $P_0/(u-p) \leq \DG < (P_h-P_0)/p-\epsilon_t$, the bargaining proceeds with the task party generates a new price offer. 
\end{itemize}

In Case 3, by offering a feature bundle closest to $(P_h-P_0)/p$ to proceed the bargaining, the data party can maximize its payment if the task party accepts the bundle due to payment is monotonic to $\DG$. The existence of Case 4 eliminates the change of the data party deliberately offering less optimal bundles and waits for a higher quoted price of the task party. 

\begin{algorithm}

\caption{Perfect Performance Information Bargaining}
\algsetup{indent=1.5em}
\begin{algorithmic}[1]
\REQUIRE 
{The feature bundle under sale of the data party: $\mcF$; The performance gain of $F_i\in\mcF$: $\{\DG_i\}$}
\STATE \textcolor{gray}{/*On the side of the task party*/}
\STATE {\textbf{Initialize}: The task party targets a performance gain $\DG^*$; Initialize a base quoted price $(p^0,P^0_0,P^0_h)$ that satisfies $(P^0_h-P^0_0)/p=\DG^*$, where $p^0\leq u$ and $P^0_h\leq B$; $u$ is the utility rate of the task party; $B$ is the budget of the task party}
\WHILE {True}
\STATE \textcolor{gray}{/*On the side of both parties*/}
\STATE {Initialize base model parameters on the task party and the data party}
\STATE \textcolor{gray}{/*On the side of the task party*/}

\IF {is the first round of bargaining}
	\STATE {$(p,P_0,P_h)\leftarrow(p^0,P^0_0,P^0_h)$}
\ELSE
	\STATE $\DG\leftarrow$ conduct VFL with data party on the feature bundle $F$
	\IF {Case 4 in Section \ref{sec:per_termi}}
		\RETURN {Transaction fails}
	\ELSIF {Case 5 in Section \ref{sec:per_termi}}
		\RETURN {Transaction successes}
	\ELSE
		\STATE {Sample a finite set of quoted prices $P=\{(p^i,P^i_0,P^i_h)\}$, where $p^i \in (p^0,u]$, $P^i_h \in (P^0_h,B]$, $P^i_0=P^i_h-p^i\DG^*$, and $P^j_0\geq P^0_0$}
		\STATE {$(p,P_0,P_h) \leftarrow (p^j,P^j_0,P^j_h)=\min_{\mbp_j\in P} P^j_h$}
	\ENDIF
\ENDIF
\STATE \textcolor{gray}{/*On the side of the data party*/}
\STATE {$\mcF'\leftarrow$ Filter out feature bundles in $\mcF$ with reserved price higher than $(p,P_0)$}

\IF {Case 1 in Section \ref{sec:per_termi}}
	\RETURN {Transaction fails}
\ELSE
	\STATE {$F \leftarrow \min_{F_i\in\mcF} (P_h-P_0)/p-\DG_i$ while $(P_h-P_0)/p-\DG_i>0$}

	\IF {Case 2 in Section \ref{sec:per_termi}}
		\RETURN {Transaction successes}
	\ENDIF
\ENDIF	

\ENDWHILE
\end{algorithmic}
\label{alg:per}
\end{algorithm}

The bargaining process under perfect information setting is illustrated in Algorithm \ref{alg:per}.

\subsubsection{Perfect Performance Information: Bargaining Cost}
Note that cost exists in the proposed bargaining model: the third party can charge for query fees when the two players request model performance information; the VFL-related communication and training cost accumulates as the bargaining round continues. Therefore, it is necessary to take into account bargaining costs when discussing bargaining results. We use a cost function to describe the bargaining cost on each party, denoted as $C_t(T)$ and $C_d(T)$, where $C_t(\cdot)$ is the cost function on the task party, and $C_d(\cdot)$ is the cost function on the data party, which are of the bargaining round $T$. We assume that the bargaining cost takes effect on the final revenue received by the corresponding party and can be formulated in an additive manner. More specifically, for the task party, the final revenue it receives can be expressed as:
\begin{equation*}
R_t(T)=u \Delta G -\min\{\max\{P_0,P_0+p\Delta G\},P_h\} -C_t(T).
\end{equation*}
While for the data party, it is:
\begin{equation*}
R_d(T)=\min\{\max\{P_0,P_0+p\Delta G\},P_h\} -C_d(T).
\end{equation*}

In a bargaining model with the above-defined costs, Theorem \ref{thm:1} and Lemma \ref{le:1} still hold. This is because the cost revenue is independent of the bargaining result but only a function of the number of rounds. Therefore, the task party and the data party continue seeking to maximize the obtained revenue during a specific round of bargaining. However, as the bargaining cost monotonically increases to bargaining rounds, late-reached agreements can lead to the situation that the increased revenue does not cover the increased cost of bargaining, which makes the continuation of bargaining not necessary. In this case, the player would tend to reach the optimal solution as early as possible and terminate the negotiation when it is not worth continuing. 

Denote the current bargaining round as $T$, and the current quoted price as $(p,P_0,P_h)$. On the data party, denote that the performance gain of feature bundle $F_i$ as $\DG_i$. Suppose after selecting, the feature bundle offered by the data party gives performance gain $\DG$, while the feature bundle that gives performance gain $(P_h-P_0)/p$ is $F_j$, \ie $\DG_j=(P_h-P_0)/p$. We consider that the task party would accept the current offer instead of proceeding the bargaining if
\begin{equation}
\begin{aligned}
&P_0+p\DG_i -C_d(T)\geq  \\
&\max\{P_{0,l},P_0\}+\max\{p_l,p\}\DG_j-C_d(T+1) -\epsilon_{d,c},
\end{aligned}
\label{eq:temi_cost_data}
\end{equation}
where the left-hand-side (LHS) denotes the net revenue the data party can receive in the current round. As for the right-hand-side (RHS), $\max\{P_{0,l},P_0\}+\max\{p_l,p\}\DG_j$ denotes the lowest payment the data party can receive in the next round if it offers the feature bundle $F_i$. Equation \ref{eq:temi_cost_data} indicates that the data party should accept the current quoted price if a conservative estimation of the net revenue next round under is lower than the net revenue of the current round under tolerance $\epsilon_{d,c}$.

On the task party, denote the performance gain of the current round as $\DG$. We consider it would accept the current offer instead of proceeding the bargaining if
\begin{equation}
\begin{aligned}
&u\DG-(P_0+p\DG) -C_t(T)\geq \\
&u\frac{P_h-P_0}{p}- P_h-C_t(T+1) -\epsilon_{t,c},
\end{aligned}
\label{eq:temi_cost_task}
\end{equation}
where the LHS denotes the net profit of the current round. The RHS denotes the upper bound of revenue the task party can receive in the next round of bargaining. This can be derived from 1) as $(P_h-P_0)/p$ is the target performance gain of the task party, we have $(P_h-P_0)/p\geq \DG$; 2) since the quoted price of the next round should have a higher highest payment $P_h'$ to bargain for a higher $\DG$, we have $P_h \leq P'_h$; and 3) $C_t(T+1)>C_t(T)$. $\epsilon_{t,c}$ is a hyper-parameter controlling the tolerance of the gap. Equation \ref{eq:temi_cost_task} suggests that if the net profit of the current round is greater than the highest payment the task party could receive in the next round under tolerance $\epsilon_{t,c}$, it should accept the feature bundle.

Based on the above two formulations, we can modify the termination condition with bargaining cost For the data party,
\begin{itemize}
	\item Case 3 (with bargaining cost). If the select $F_i$ in $\mcF$ satisfies $p_{l,i}\leq p$ and $P_{l,i}\leq P_0$, while $(P_h-P_0)/p-\DG_i > \epsilon_d$: 1) if Equation \ref{eq:temi_cost_data} holds, the bargaining terminates by the task party with transaction successes; 2) else, the bargaining proceeds with the data party offering the feature bundle.
\end{itemize}

And for the task party,
\begin{itemize}
	\item Case 6 (with bargaining cost). If the received $\DG$ satisfies $P_0/(u-p) \geq \DG < (P_h-P_0)/p-\epsilon_t$: 1) if Equation \ref{eq:temi_cost_task} holds, the bargaining terminates by the task party with transaction successes; 2) else, the bargaining proceeds with the task party generates a new feature offer.
\end{itemize}

The modified two cases take into account that the bargaining processes with cost. We further propose that the two cases can be considered as a generalization of transaction success condition in Case 2 and Case 5 of Section \ref{sec:per_termi}. 

\begin{proposition}
When the bargaining cost is a constant value, Equation \ref{eq:temi_cost_data} is a reformulation of $(P_h-P_0)/p-\DG_i > \epsilon_d$ in Case 3 of Section \ref{sec:per_termi}. 
\end{proposition}
\begin{proof}
When the bargaining cost is constant, Equation \ref{eq:temi_cost_data} can be re-written as
\begin{equation*}
P_0+p\DG_i \geq
\max\{P_{0,l},P_0\}+\max\{p_l,p\}\DG_j -\epsilon_{d,c}.
\end{equation*}
For the RHS of the above formulation, we have
\begin{equation*}
\begin{aligned}
RHS &\geq P_0+p\DG_j -\epsilon_{d,c}=P_0+p\frac{P_h-P_0}{p} -\epsilon_{d,c}\\&=P_h-\epsilon_{d,c}.
\end{aligned}
\end{equation*}
We can then re-formulate Equation \ref{eq:temi_cost_data} as 
\begin{equation*}
\frac{P_h-P_0}{p}-\DG_i \leq \epsilon_{d},
\end{equation*}
where 
\begin{equation*}
\epsilon_{d}=\big(\epsilon_{d,c}-(\max\{P_{0,l},P_0\}+\max\{p_l,p\}\frac{P_h-P_0}{p}-(P_h))\big)/p.
\end{equation*}

\end{proof}

\begin{proposition}
When the bargaining cost is a constant value, Equation \ref{eq:temi_cost_task} is a reformulation of $\DG \geq (P_h-P_0)/p-\epsilon_t$ in Case 4 of Section \ref{sec:per_termi}. 
\end{proposition}
\begin{proof}
When the bargaining cost is constant, Equation \ref{eq:temi_cost_task} can be re-written as
\begin{equation*}
u\DG-(P_0+p\DG) \geq u\frac{P_h-P_0}{p}- P_h-\epsilon_{t,c}.
\end{equation*}
Subtract $P_0$ from both sides and combine terms with the same factors we have 
\begin{equation*}
(u-p)\DG\geq(u-p)\frac{P_h-P_0}{p}-\epsilon_{t,c}.
\end{equation*}
We can reformulate the above equation as 
\begin{equation*}
\DG \geq (P_h-P_0)/p-\epsilon_t,
\end{equation*}
where $\epsilon_t=\epsilon_{t,c}/(u-p)$.
\end{proof}

\subsection{Extending to Imperfect Performance Information}

We then extend the bargaining to imperfect performance information setting, where the performance gain of a feature bundle is not known in advance to both parties due to possible reasons. In this situation, the bargaining offers should be made in an estimation fashion. 

\subsubsection{Performance Estimation}
We suppose that there is an estimation function for each party. For the data party, we denote the estimation function as $g$, which takes feature bundles as input and predicts performance gain as output. Mathematically, the optimal parameters of the estimation function can be expressed as:

\begin{equation}
\theta_g^*=\min_{\theta_g}\mcL(g(F;\theta_g),\DG),
\label{eq:est_per_data}
\end{equation}
where $\theta_g$ is the model parameter of $g$, $\mcL$ evaluates the loss on the training instances, and $\DG$ is the groundtruth performance gain.

While for the task party, we propose to estimate the performance gain of quoted price instead since feature bundle related information is preserved by the data party. Denote the performance estimation function as $f$, it uses quoted price as input and predicts performance gain, meaning that the task party learns about how much performance gain it would obtain by offering a certain price. The objective of the network can be formulated as follows,
\begin{equation}
\theta_f^*=\min_{\theta_f}\mcL(f(p,P_0,P_h;\theta_f),\DG),
\label{eq:est_per_task}
\end{equation}
where $\theta_f$ is the model parameter of $f$.

Taking $f$ and $g$ as a whole, the learning process can be described as finding optimal $<(p,P_0,P_h),\DG>$ matching between the two parties, which is inconsistent with the perfect performance information setting. Accurate performance estimation is crucial for both the task and the data parties in such a situation, as it directly influences the net profit maximization of the task party and the monetary payoff optimization of the data party. However, the training of such a function faces several challenges. First, \textbf{training while bargaining}. The obtaining of $\theta^*_g$ and $\theta^*_f$ requires a set of ($(p,P_0,P_h)$, $F$, $\DG$) samples, which do not exist until VFL is performed between the two parties, while VFL courses only happen when bargaining happens between the two parties. It is important to simultaneously achieve both effective bargaining and effective training. Second, \textbf{simultaneous update}. Once a task party and a data party are paired for VFL, they both need to train their estimation functions so as to fit the current partner. Such a simultaneous training can introduce dynamics and may lead to convergence problems if there is no mutual understanding shared. To solve these problems, we propose the following offer-generating strategies on the side of the data party and the task party.

\subsubsection{Imperfect Performance Information: Bargaining Analysis On The Data Party}

Similar to the perfect performance information setting, the task party's feature bundle offer should be made based on the given quoted price received from the task party. Denoted as $(p,P_0,P_h)$, the maximum revenue that could be received by the data party is $P_h$. The data party's goal is to identify a feature bundle whose predicted performance gain is as close as possible to, but does not exceed, the quantity $(P_h - P_0)/p$. Similar to the setting of perfect performance information, this ensures that the selected bundle generates a performance gain that maximizes the payment within the price range of $(P_0, P_h)$. However, in the imperfect performance information setting, the data party does not know which feature bundle satisfies such a condition. Its bargaining strategy should thus be made based on estimation.

The data party first filters out feature bundles that have higher reserved prices than the given $(p,P_0)$. Then, for the rest of the bundles, it makes predictions on the performance gain of each bundle based on $g$. With the estimated performance, the data party selects the feature bundle that is expected to have the closest performance gain to $(P_h - P_0)/p$ as the bargaining offer. The estimation function is updated after VFL is performed between the two parties and the real performance gain is calculated. 


\subsubsection{Imperfect Performance Information: Bargaining Analysis On The Task Party}

For the task party, we propose the following bargaining strategy.

Initially, the task party targets certain performance gain and generates a large enough set of quoted prices that conform to the constraints outlined in Equation \ref{eq:cri}, while taking into account the budget and the utility rate, similar to in the perfect performance information setting. Subsequently, it employs $g$ to predict the performance gain associated with each sampled price in the set. When making a bargaining offer, the task party filters out prices whose predicted performance gain is no greater than or equal to $(P_h-P_0)/p-\epsilon_t$ and selects the one with the highest net profit to offer. If the filtered set is empty, it selects the one with the highest net profit. It's noteworthy that the task party's estimation function specifically considers quoted prices that conform to the constraints outlined in Equation \ref{eq:cri}. Consequently, it focuses the attention on prices that are consistent with the equilibrium price specified in Theorem \ref{thm:1} and Lemma \ref{le:1} when updating the performance estimation function. After the feature bundle offer is selected from the data party, the two parties perform VFL, and the obtained \textit{real performance gain} is used to update the performance estimation function $g$.

\subsubsection{Imperfect Performance Information: Termination Conditions}
\label{sec:imper_termi}
We denote the data party's estimated performance gain of feature bundle $F_i$ as $g(F_i)$, the quoted price offered by the task party as $(p,P_0,P_h)$, and the estimated performance gain as $f(p,P_0,P_h)$, and the real performance gain received by the task party as $\DG$. The bargaining under imperfect information condition can be divided into the following cases.

On the side of the data party,
\begin{itemize}
\item Case I: If there is no feature bundle $F_i$ in $\mcF$ that satisfies $p_{l,i}\leq p$ and $P_{l,i}\leq P_0$, the bargaining terminates by the data party with the transaction fails, unless Case VII is met.

\item Case II: If the selected $F_i$ satisfies $p_{l,i}\leq p$ and $P_{l,i}\leq P_0$, and 1) $(P_h-P_0)/p-g(F_i)\leq \epsilon_d$, or 2) $(P_h-P_0)/p>\max{\{g(F_i)\}}$, or 3) $(P_h-P_0)/p<\min{\{g(F_i)\}}$, the bargaining terminates by the data party with transaction successes and 1) $F_i$, or 2) $F_{max}$, or 3) $F_{min}$ offered, unless Case VII is met.

\item Case III: If the selected $F_i$ satisfies $p_{l,i}\leq p$ and $P_{l,i}\leq P_0$, while Case II is not met, the bargaining proceeds with $F_i$ offered.
\end{itemize}

On the side of the task party,
\begin{itemize}

\item Case IV: If the received $\DG$ satisfies $\DG < P_0/(u-p)$, the bargaining terminates by the task party with the transaction fails, unless Case VII is met.

\item Case V: If the received $\DG$ satisfies $\DG \geq (P_h-P_0)/p-\epsilon_t$, the bargaining terminates by the task party with transaction successes, unless Case VII is met.

\item Case VI: If the received $\DG$ satisfies $P_0/(u-p) \leq \DG < (P_h-P_0)/p-\epsilon_t$, the bargaining proceeds with the task party generates a new price offer. 

\end{itemize}

On the side of both parties,

\begin{itemize}
\item Case VII: If the bargaining is within the initial $N$ rounds, the bargaining proceeds with the data party operates on Case III, and the task party operates on Case VI.
\end{itemize}

The key difference between the termination condition of imperfect and perfect performance information settings includes 1) the action taken by the data party is based on estimated performance gain; 2) the transaction failing conditions, \ie Case 1 and Case 3 in Section \ref{sec:per_termi}, are relaxed in the first $N$ rounds of bargaining due to the estimation functions of the two parties are not well trained. Note that for the task party, the performance gain estimation function takes effect on price offer generating stage, while the termination conditions are based on the calculated real performance gain after receiving feature bundle and performing VFL. Moreover, despite in the setting of imperfect performance information, the task party lacks prior knowledge of the potential range of performance gains that could be offered by the data party, the practice of targeting a specific performance gain, as done in the case of perfect performance information, becomes still doable by introducing Case II. 

\subsection{Security Analysis}
In federated learning, there is a risk of malicious participants trying to manipulate the model or extract information from the training process. This risk is especially of concern in the context of VFL as it often requires direct communication between two parties without a server \cite{liu2022vertical}. Consequently, defenses and attacks are widely studied in the literature \cite{jiang2022comprehensive, luo2021feature, jin2021cafe, lyu2022privacy}. We here analyze possible threat models of the VFL market from two perspectives: 1) VFL training and inference and 2) the bargaining. \textit{FL Training and Inference.} The proposed VFL market is FL protocol-agnostic since it only takes the output (performance gain) of a VFL course as intermediate information for subsequent bargaining but does not affect and is not affected by the specific training protocol of VFL. This makes it adaptable to any security-related VFL training protocols and does not introduce a new threat model during the training and inference phase. \textit{Bargaining.} As performance gain is exchanged between the two parties, a party can access this information and conduct possible inference attacks on the other party's data. To eliminate this risk, encryption methods such as Homomorphic Encryption (HE) \cite{paillier1999public} and Secure Multi-party Computation (SMC) \cite{yao1982protocols} can be adopted for multiplication or comparing related operations. As the detailed design of security-preserving approaches is out of the scope of the paper, we omit it here. 


\begin{figure*} [t]
\subfigure[On Titanic]{
\centering
\includegraphics[width=0.19\linewidth]{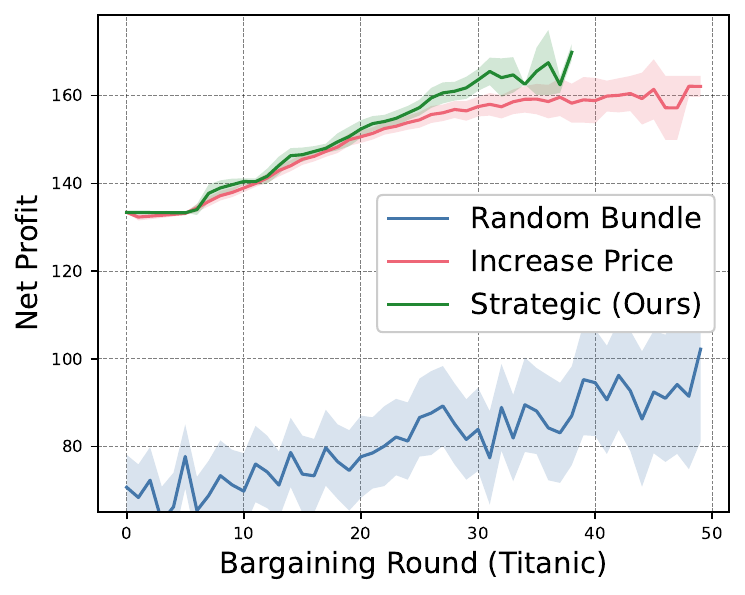}
}%
\subfigure[On Titanic]{
\centering
\includegraphics[width=0.19\linewidth]{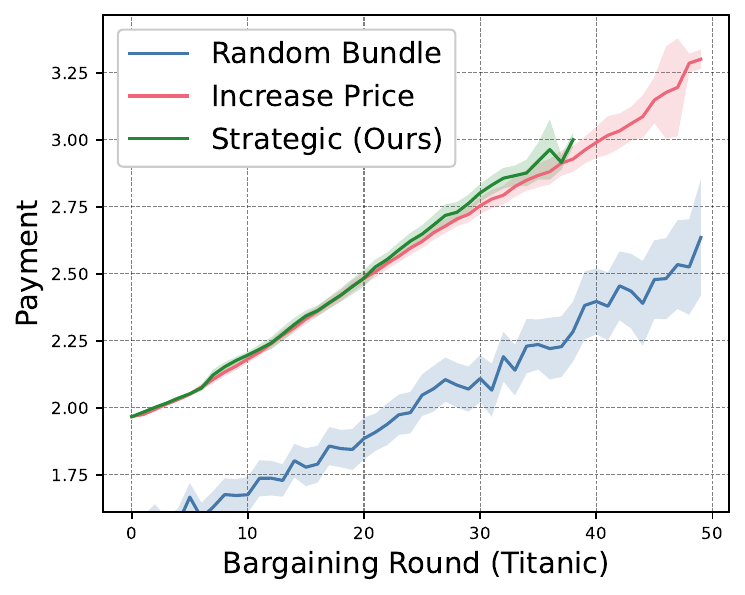}
}%
\subfigure[On Titanic]{
\centering
\includegraphics[width=0.19\linewidth]{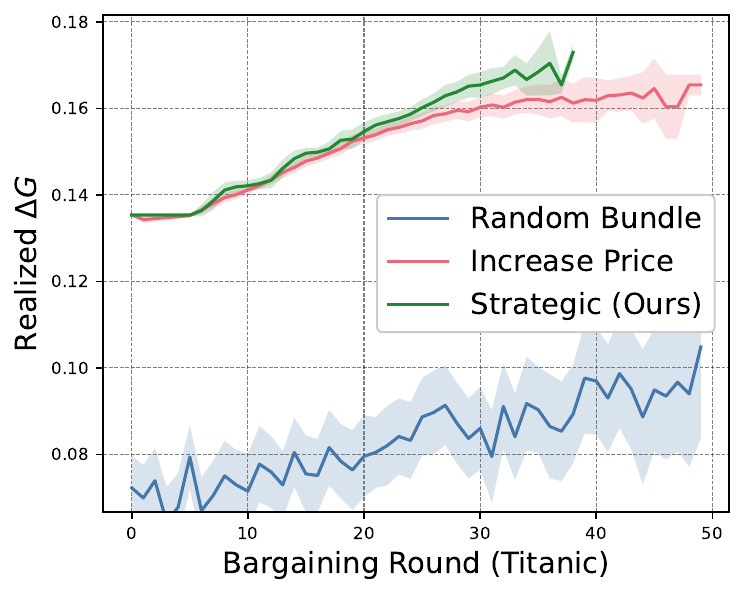}
}%
\subfigure[On Titanic]{
\centering
\includegraphics[width=0.19\linewidth]{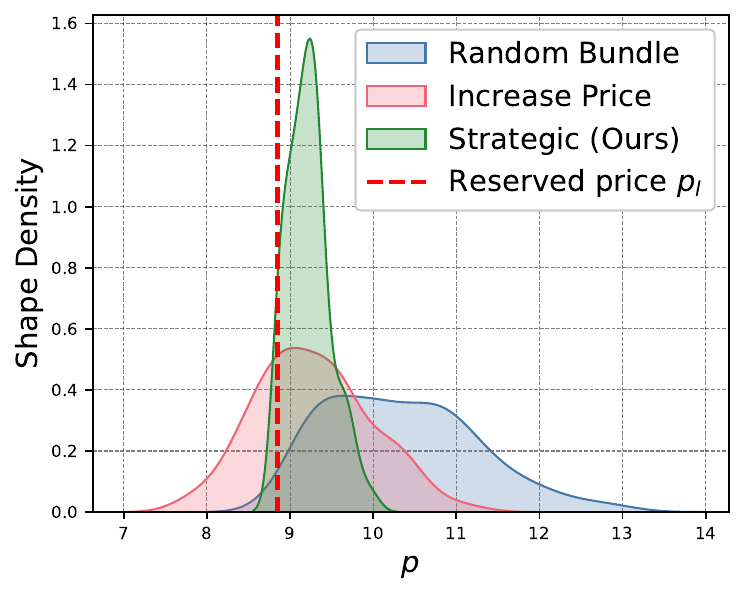}
}%
\subfigure[On Titanic]{
\centering
\includegraphics[width=0.19\linewidth]{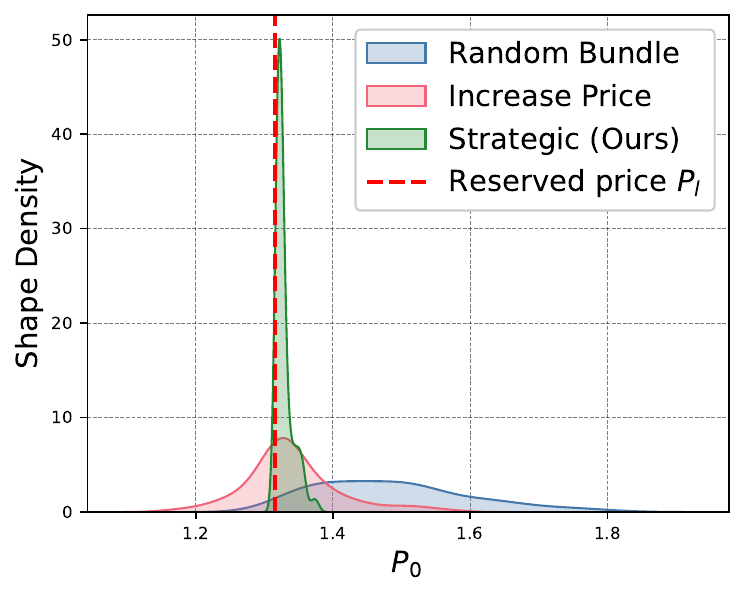}
}%
\centering

\subfigure[On Credit]{
\centering
\includegraphics[width=0.19\linewidth]{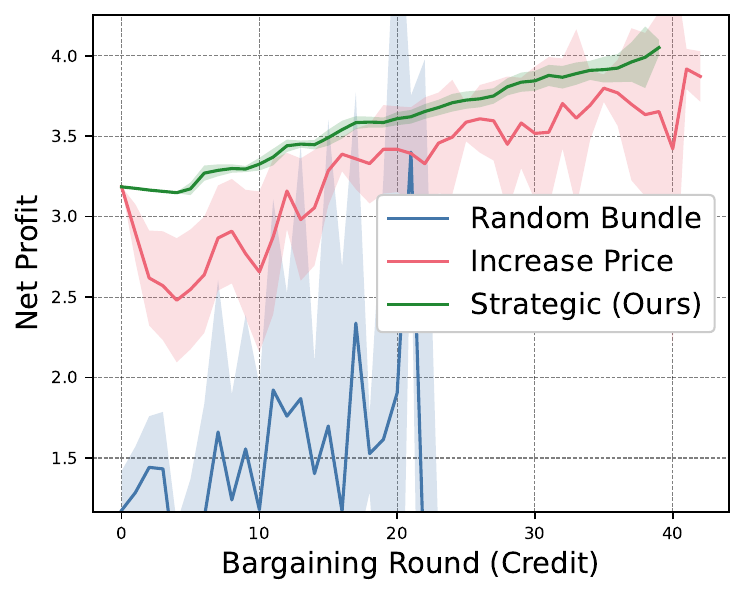}
}%
\subfigure[On Credit]{
\centering
\includegraphics[width=0.19\linewidth]{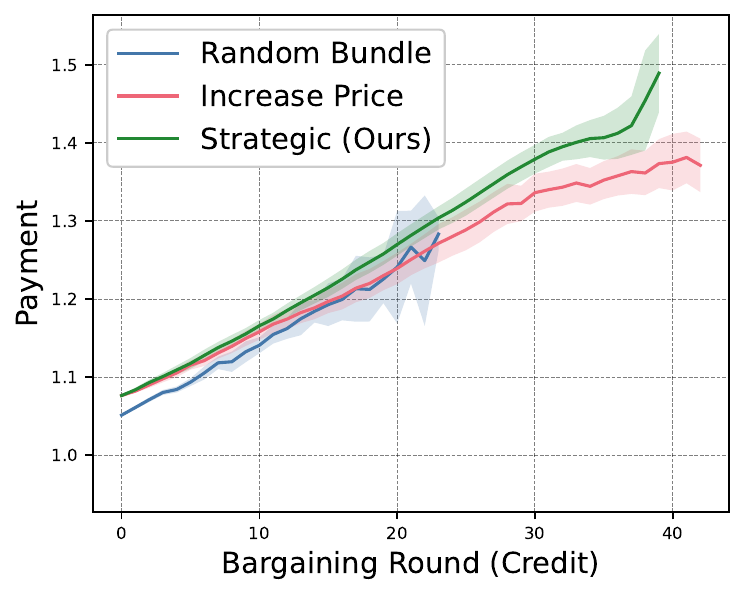}
}%
\subfigure[On Credit]{
\centering
\includegraphics[width=0.19\linewidth]{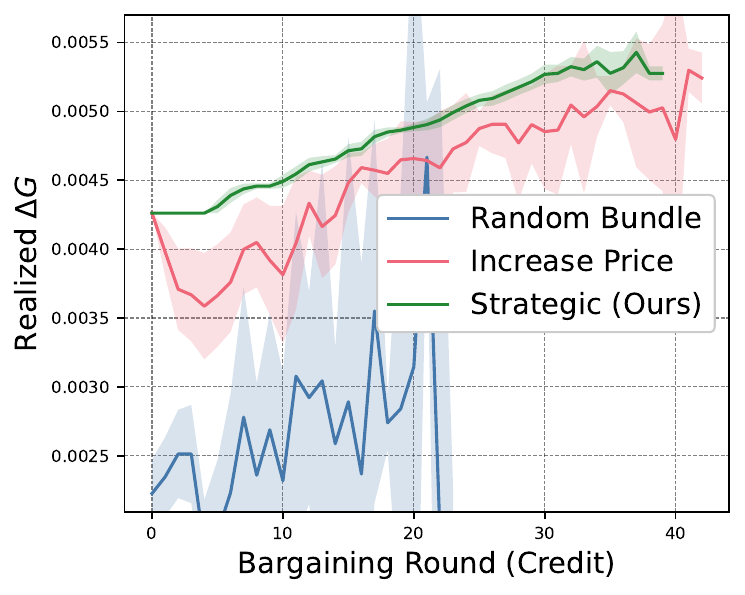}
}%
\subfigure[On Credit]{
\centering
\includegraphics[width=0.19\linewidth]{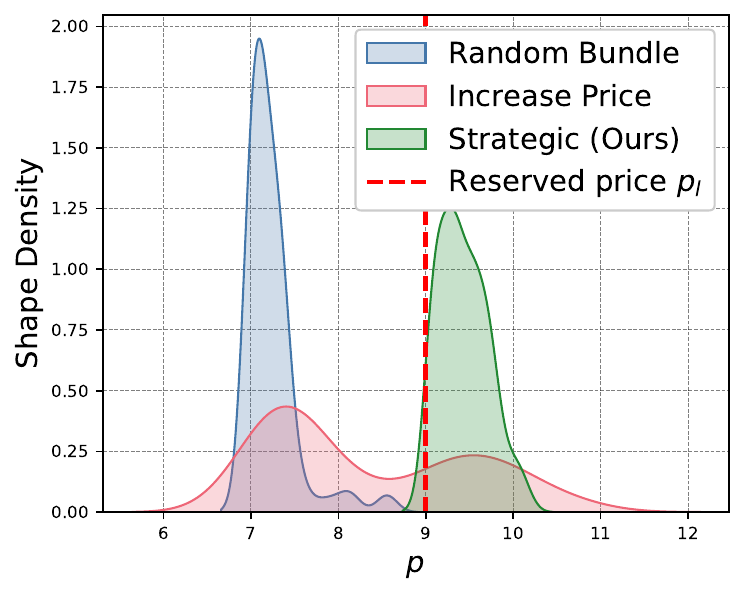}
}%
\subfigure[On Credit]{
\centering
\includegraphics[width=0.19\linewidth]{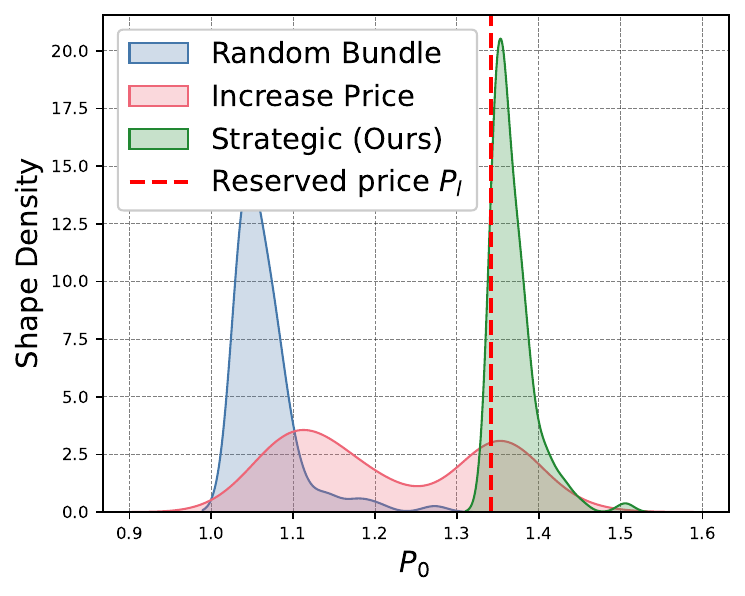}
}%
\centering

\subfigure[On Adult]{
\centering
\includegraphics[width=0.19\linewidth]{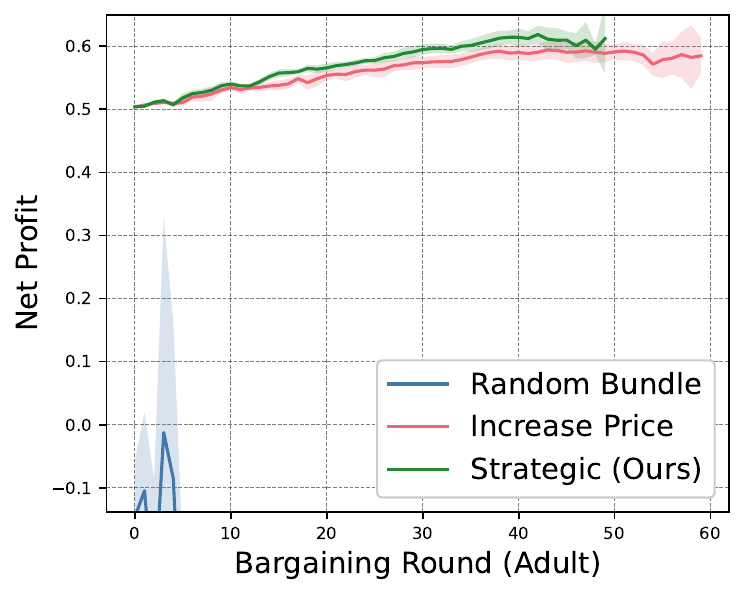}
}%
\subfigure[On Adult]{
\centering
\includegraphics[width=0.19\linewidth]{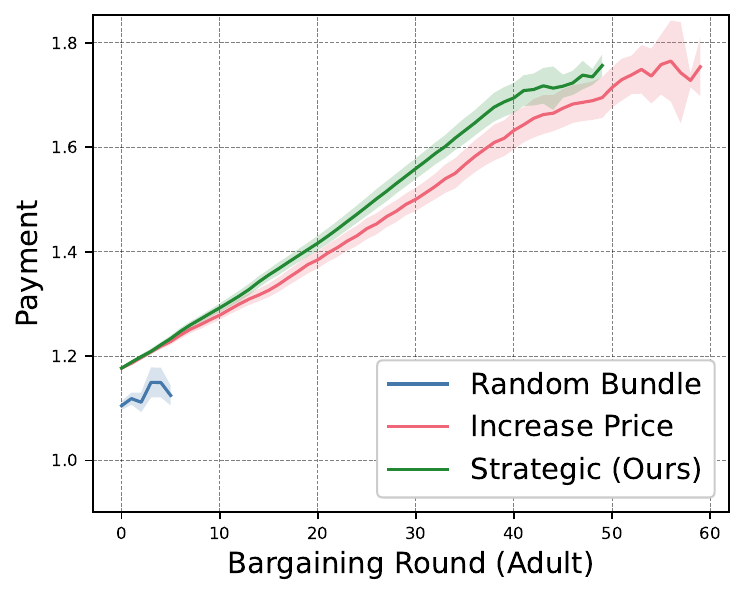}
}%
\subfigure[On Adult]{
\centering
\includegraphics[width=0.19\linewidth]{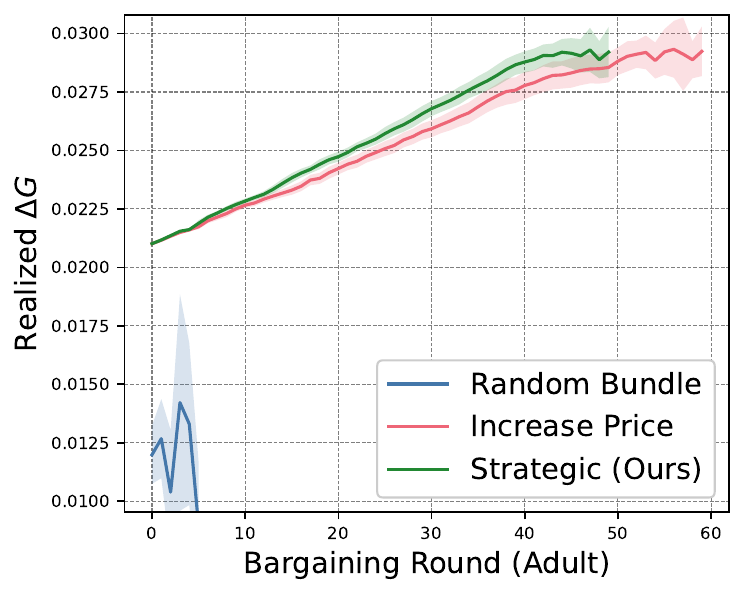}
}%
\subfigure[On Adult]{
\centering
\includegraphics[width=0.19\linewidth]{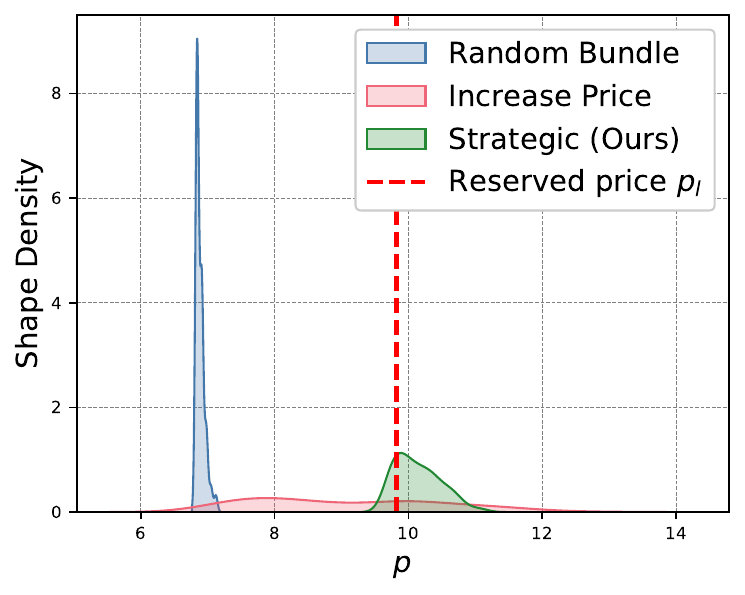}
}%
\subfigure[On Adult]{
\centering
\includegraphics[width=0.19\linewidth]{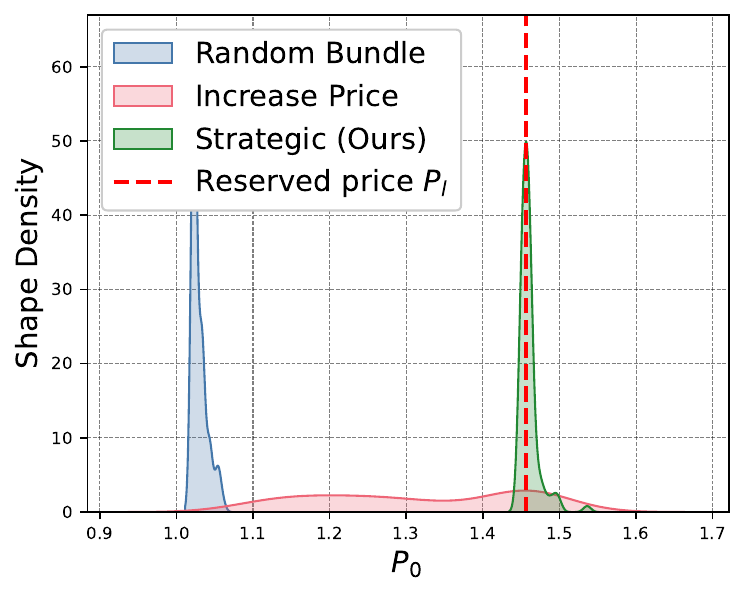}
}%
\centering
\caption{Bargaining results on the Titanic, Credit, and Adult datasets with Random Forest model. (a) Net profit \vs bargaining rounds. (b) Payment \vs bargaining rounds. (c) $\DG$ \vs bargaining rounds. (e) Final $p$ of each run of bargaining game \vs the reserved price $p_l$ of the data parties target feature bundle. (f) Final $P_0$ of each run of bargaining game \vs the reserved price $P_l$ of the data parties target feature bundle.}
\label{fig:per_false}
\end{figure*}
\section{Experiment}

\begin{figure*} [t]
\subfigure[On Titanic]{
\centering
\includegraphics[width=0.19\linewidth]{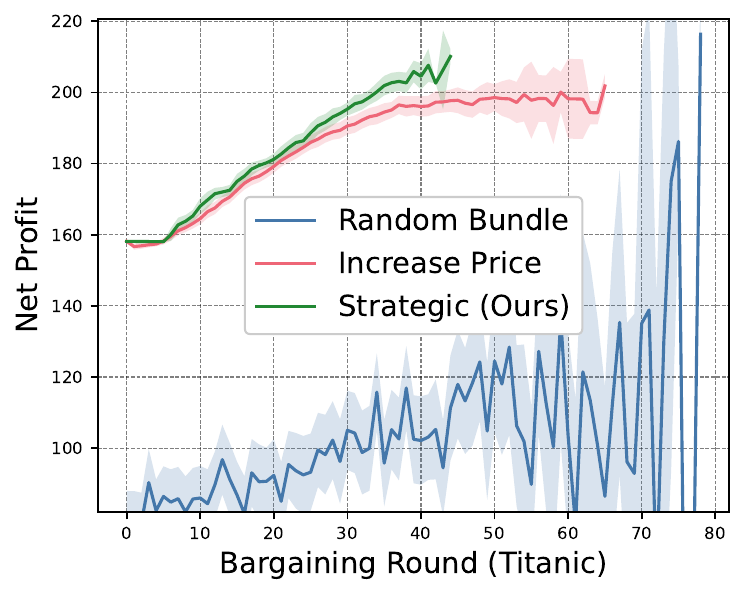}
}%
\subfigure[On Titanic]{
\centering
\includegraphics[width=0.19\linewidth]{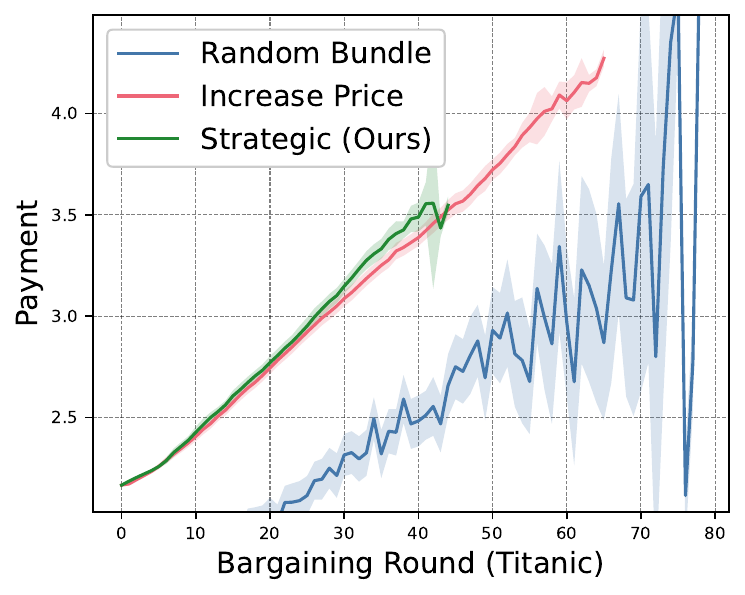}
}%
\subfigure[On Titanic]{
\centering
\includegraphics[width=0.19\linewidth]{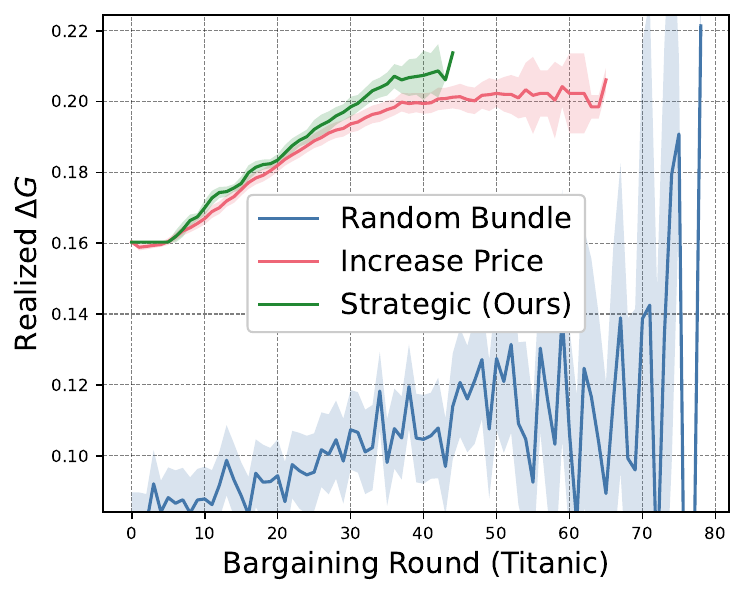}
}%
\subfigure[On Titanic]{
\centering
\includegraphics[width=0.19\linewidth]{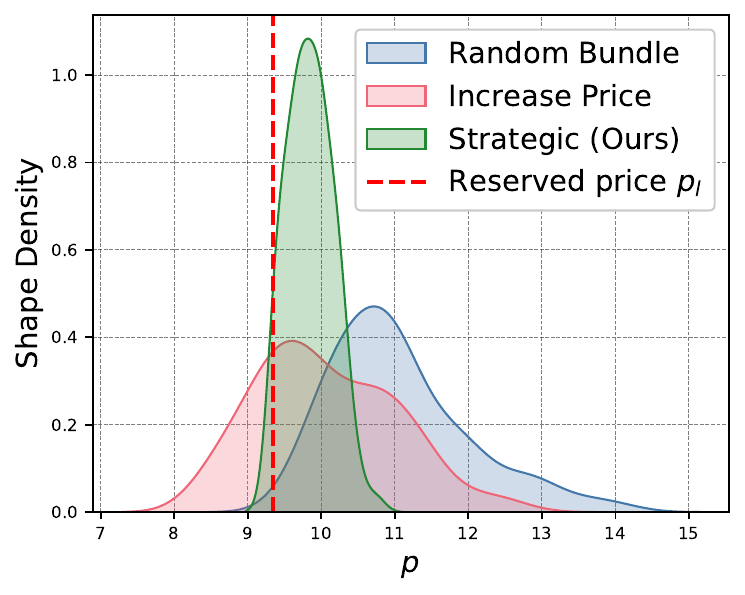}
}%
\subfigure[On Titanic]{
\centering
\includegraphics[width=0.19\linewidth]{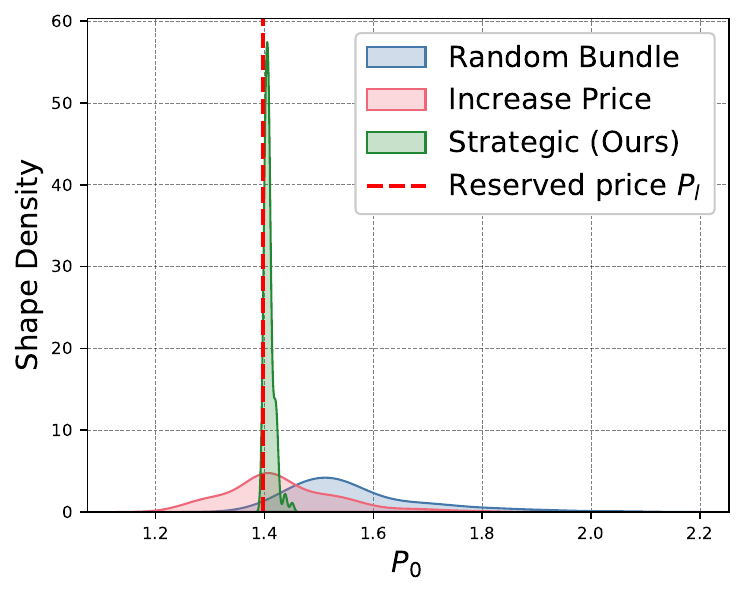}
}%
\centering

\subfigure[On Credit]{
\centering
\includegraphics[width=0.19\linewidth]{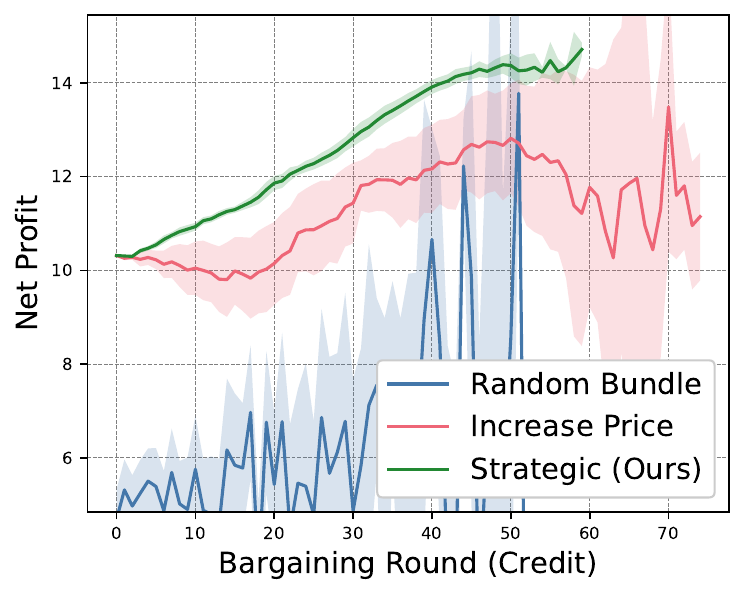}
}%
\subfigure[On Credit]{
\centering
\includegraphics[width=0.19\linewidth]{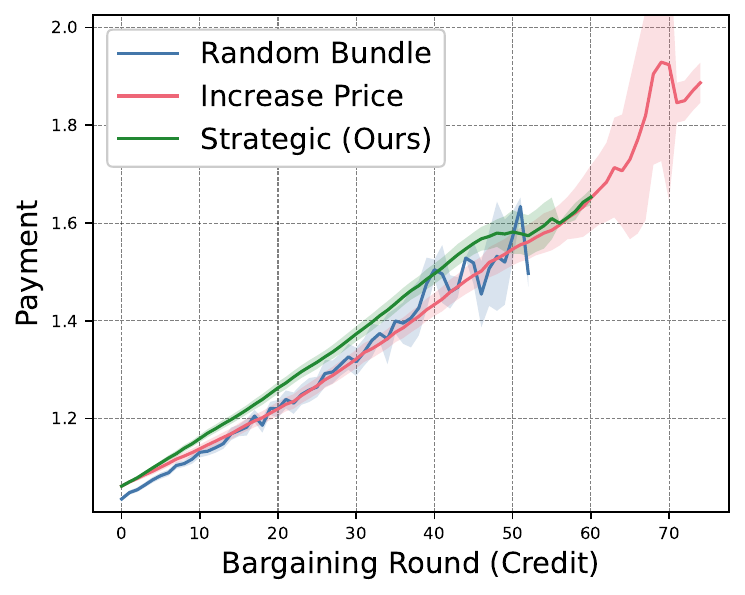}
}%
\subfigure[On Credit]{
\centering
\includegraphics[width=0.19\linewidth]{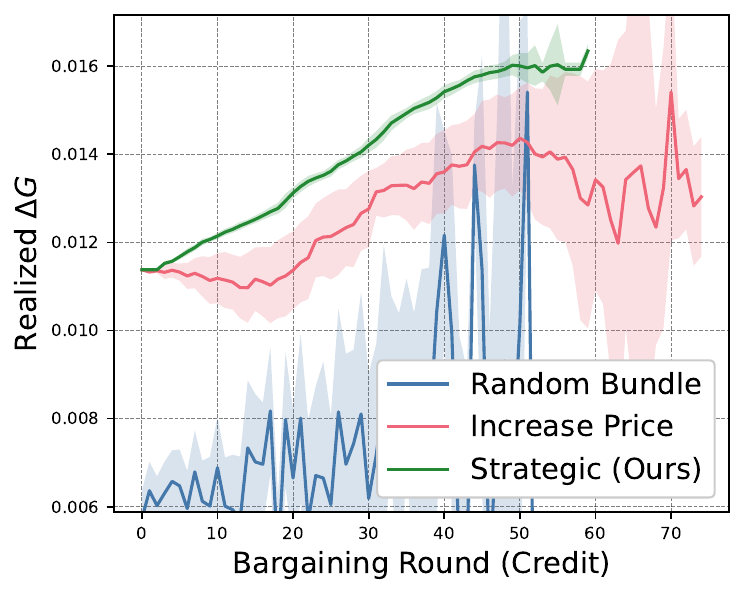}
}%
\subfigure[On Credit]{
\centering
\includegraphics[width=0.19\linewidth]{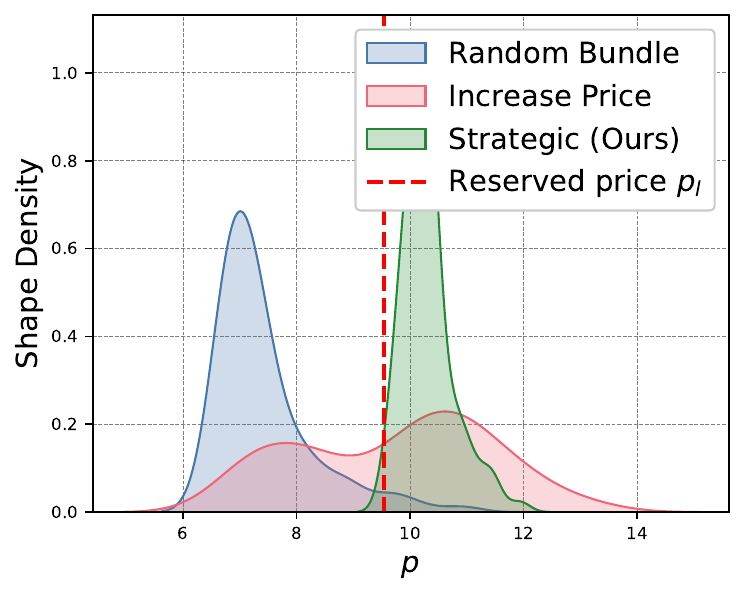}
}%
\subfigure[On Credit]{
\centering
\includegraphics[width=0.19\linewidth]{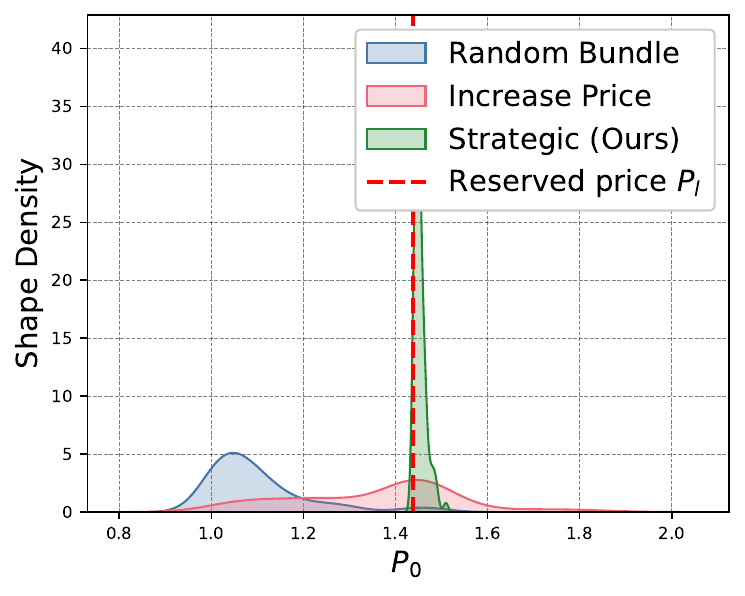}
}%
\centering

\subfigure[On Adult]{
\centering
\includegraphics[width=0.19\linewidth]{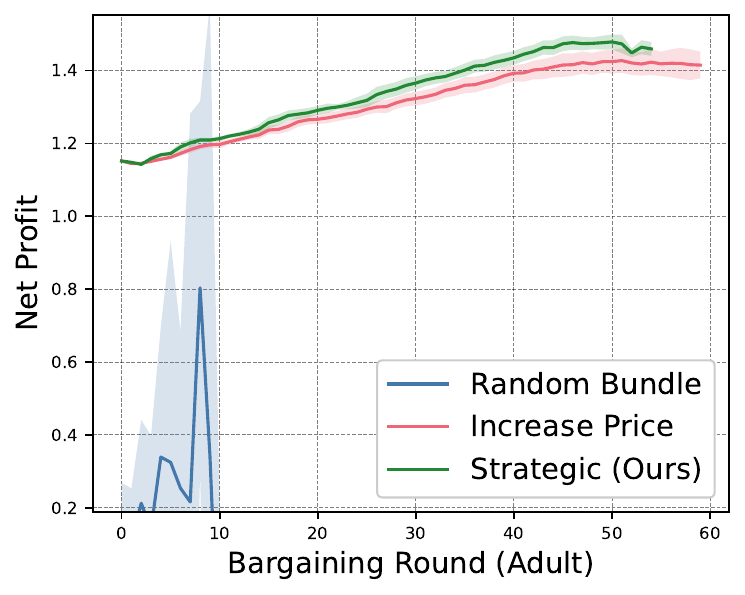}
}%
\subfigure[On Adult]{
\centering
\includegraphics[width=0.19\linewidth]{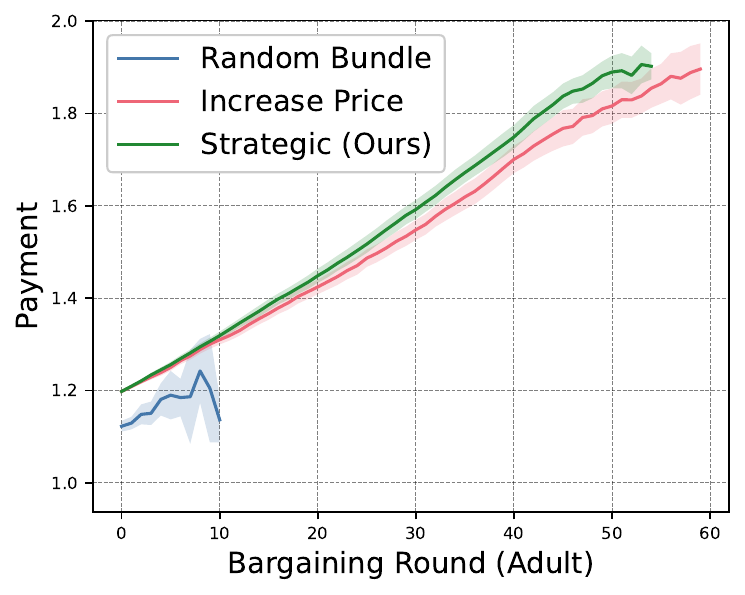}
}%
\subfigure[On Adult]{
\centering
\includegraphics[width=0.19\linewidth]{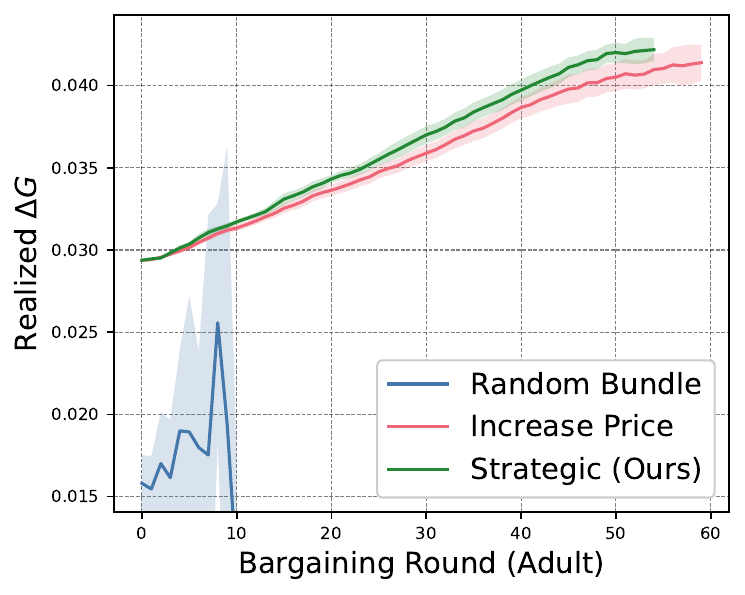}
}%
\subfigure[On Adult]{
\centering
\includegraphics[width=0.19\linewidth]{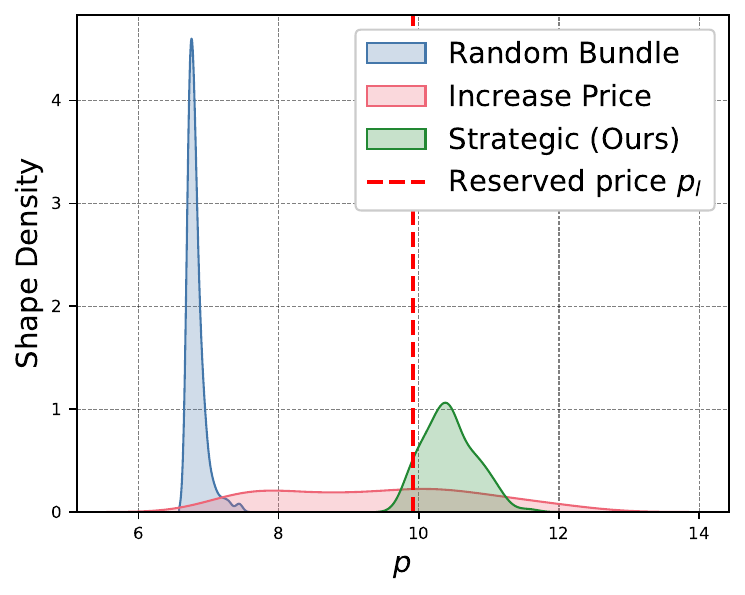}
}%
\subfigure[On Adult]{
\centering
\includegraphics[width=0.19\linewidth]{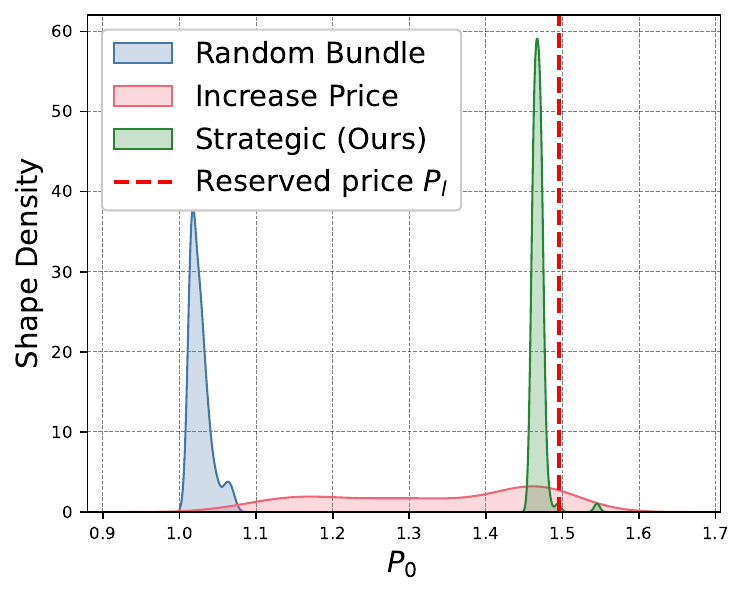}
}%
\centering
\caption{Bargaining results on the Titanic, Credit, and Adult datasets with 3-layer MLP model. Similar subfigure format as Figure \ref{fig:per_false}.}
\label{fig:per_true}
\end{figure*}

\subsection{Experiment Setup}
All experiments are implemented on 2 Intel(R) Xeon(R) Gold 6248R CPU @ 3.00GHz and 8 NVIDIA A100.

\subsubsection{Datasets and Metrics}
\begin{table}[t]
\centering
\caption{Dataset statistics.}
\fontsize{8}{9.5}\selectfont
\begin{tabular}{@{}llll@{}}
\specialrule{0.3mm}{0em}{0em} 
Datasets                              & Titanic & Credit & Adult \\ \hline
\# samples                            & 891     & 30000  & 48842 \\
original \# features (total)          & 11      & 25     & 14    \\
preprocessed \# features (task party) & 10      & 9      & 52    \\
preprocessed \# features (data party) & 19      & 21     & 36    \\ \specialrule{0.3mm}{0em}{0em} 
\end{tabular}
\label{tb:data}
\end{table}
We evaluate the proposed bargaining model on three real-world tabular datasets. 

\textit{Titanic} \footnote{\url{https://www.kaggle.com/competitions/titanic/data}}: The Titanic dataset contains information about passengers aboard the RMS Titanic, including 11 features such as age, sex, passenger class, and ticket fare. The dataset is commonly used for binary classification tasks, predicting whether a passenger survived or not based on the provided features.

\textit{Credit} \footnote{\url{https://www.kaggle.com/datasets/uciml/default-of-credit-card-clients-dataset}}: The dataset contains credit card client data from Taiwan between April 2005 and September 2005. It includes 25 variables such as ID, credit limit, gender, education, marital status, age, repayment status, bill statement amounts, and previous payment amounts. The prediction task is to classify if a client will default on his/her payment in the next month.

\textit{Adult} \footnote{\url{https://archive.ics.uci.edu/dataset/2/adult}}: The Adult dataset, \aka the ``Census Income" dataset, contains demographic and income-related information about individuals. It includes 14 features such as age, education level, marital status, and occupation. The goal is to predict whether an individual earns more than \$50,000 per year.

We convert the multi-class categorical features in the original datasets into indicator features and then split the features into task-party-owned and data-party-owned. Note that indicator features of the same original feature are on the same party. A summarization of the original and preprocessed datasets is listed in Table \ref{tb:data}.

For the three datasets, we calculate Accuracy as the performance of the base model.

\subsubsection{Implementation Details}
For the VFL-based prediction task, we evaluate two types of base model implementations: tree-based and deep neural network (DNN) based models. For the tree-based models, we consider the two parties together training a Random Forest model, with gini index as the splitting metric. For the DNN-based model, we consider the two parties collaboratively training a 3-layer multi-layer perception (MLP), with embedding dimensions 64 and 32. For the 3-layer MLP, the learning rate is set as 1e-2, and the batch size is set as 128 for the Titanic dataset, and 512 for the Credit dataset and the Adult dataset. In the isolated training of the task party, we set the number of epochs as 200, which ensures convergence. The maximum number of bargaining rounds is set as 500, and the bargaining is considered as terminated with transaction fails if it exceeds the threshold. Other key parameters regarding bargaining are empirically studied in subsequent sections.

\subsection{Main Results}

In order to study the effectiveness of the proposed bargaining model, we conduct experiments by comparing with two non-strategic variants of the model, named Increase Price and Random Bundle. In Increase Price, the task party does not follow the criterion in Equation \ref{eq:cri} but arbitrarily increases the quoted price in each round. In Random Bundle, the data party filters out feature bundles that have a higher reserved price than the quoted price, and then arbitrarily offers one feature bundle to the task party. We set the initial status, \ie the initial quoted price and target $\DG$ on the task party side, as the same for the three models and run the bargaining 100 times. The results of mean and 95\% confidence interval of net profit received by the task party, revenue reward received by the data party (\aka the payment made by the task party, denoted as Payment in the figures), and the $\DG$ over bargaining rounds are depicted in the left three columns of Figure \ref{fig:per_false} and Figure \ref{fig:per_true}, in which Figure \ref{fig:per_false} the base model for VFL is Random Forest and that in Figure \ref{fig:per_true} is 3-layer MLP. In the right two columns of Figure \ref{fig:per_false} and Figure \ref{fig:per_true}, we present the density functions that describe the distribution of the final quoted prices when the bargaining terminates compared with the reserved price of the data party. The proposed bargaining model is denoted as Strategic.

\begin{table*}[t]
\caption{Effect of bargaining cost.}
\fontsize{8}{8}\selectfont
\begin{tabular}{@{}llcccccccc@{}}
\toprule
\multicolumn{2}{c}{Dataset}                                                      & \multicolumn{8}{c}{Titanic}                                                                                                                           \\ \midrule
\multicolumn{2}{c|}{Termination Threshold}                                       & \multicolumn{4}{c}{{\ul $\epsilon=1e-3$}}                                & $\epsilon=1e-2$ & \multicolumn{3}{c}{}                                     \\
\multicolumn{2}{c|}{Bargaining Cost}                                             & Net Profit    & Payment         & Realized $\DG$ (1e-2) & C(T)           & Net Profit      & Payment         & Realized $\DG$ (1e-2) & C(T)           \\ \midrule
\multicolumn{2}{c|}{No cost}                                                     & 2.93$\pm$0.04 & 170.0$\pm$0.04  & 0.17$\pm$0.01         & -              & 2.70$\pm$0.03   & 162.71$\pm$0.03 & 0.17$\pm$0.01         & -              \\
\multicolumn{1}{l|}{\multirow{2}{*}{$C(T)=aT$}}  & \multicolumn{1}{l|}{$a=0.1$}  & 2.91$\pm$0.04 & 170.02$\pm$0.04 & 17.29$\pm$0.00        & 3.05$\pm$0.25  & 2.71$\pm$0.03   & 162.71$\pm$0.03 & 16.54$\pm$0.00        & 2.53$\pm$0.24  \\
\multicolumn{1}{l|}{}                            & \multicolumn{1}{l|}{$a=1$}    & 2.81$\pm$0.04 & 162.60$\pm$0.04 & 16.54$\pm$0.00        & 23.52$\pm$2.79 & 2.75$\pm$0.04   & 162.67$\pm$0.04 & 16.54$\pm$0.00        & 20.67$\pm$2.77 \\
\multicolumn{1}{l|}{\multirow{2}{*}{$C(T)=a^T$}} & \multicolumn{1}{l|}{$a=1.01$} & 2.98$\pm$0.05 & 169.96$\pm$0.05 & 17.29$\pm$0.00        & 1.40$\pm$0.04  & 2.72$\pm$0.03   & 162.69$\pm$0.03 & 16.54$\pm$0.00        & 1.29$\pm$0.03  \\
\multicolumn{1}{l|}{}                            & \multicolumn{1}{l|}{$a=1.1$}  & 2.76$\pm$0.02 & 162.58$\pm$0.75 & 16.53$\pm$0.08        & 13.73$\pm$2.16 & 2.70$\pm$0.02   & 162.72$\pm$0.02 & 16.54$\pm$0.00        & 9.84$\pm$1.74  \\ \midrule
\multicolumn{2}{c|}{Dataset}                                                     & \multicolumn{8}{c}{Credit}                                                                                                                            \\ \midrule
\multicolumn{2}{c|}{Termination threshold}                                       & \multicolumn{4}{c}{{\ul $\epsilon=1e-5$}}                                & \multicolumn{4}{c}{$\epsilon=1e-4$}                                        \\
\multicolumn{2}{c|}{Bargaining cost}                                             & Net Profit    & Payment         & Realized $\DG$ (1e-2) & C(T)           & Net Profit      & Payment         & Realized $\DG$ (1e-2) & C(T)           \\ \midrule
\multicolumn{2}{c|}{No cost}                                                     & 1.42$\pm$0.02 & 4.06$\pm$0.02   & 0.55$\pm$0.01         & -              & 1.42$\pm$0.02   & 4.06$\pm$0.02   & 0.55$\pm$0.01         & -              \\
\multicolumn{1}{l|}{\multirow{2}{*}{$C(T)=aT$}}  & \multicolumn{1}{l|}{$a=0.1$}  & 1.40$\pm$0.01 & 3.92$\pm$0.08   & 0.53$\pm$0.01         & 3.31$\pm$0.27  & 1.41$\pm$0.00   & 3.90$\pm$0.07   & 0.53$\pm$0.01         & 3.25$\pm$0.29  \\
\multicolumn{1}{l|}{}                            & \multicolumn{1}{l|}{$a=1$}    & 1.20$\pm$0.01 & 3.47$\pm$0.01   & 0.47$\pm$0.00         & 9.28$\pm$1.69  & 1.19$\pm$0.02   & 3.48$\pm$0.02   & 0.47$\pm$0.01         & 7.66$\pm$1.25  \\ \cmidrule(r){1-2}
\multicolumn{1}{l|}{\multirow{2}{*}{$C(T)=a^T$}} & \multicolumn{1}{l|}{$a=1.01$} & 1.43$\pm$0.02 & 4.05$\pm$0.02   & 0.55$\pm$0.00         & 1.39$\pm$0.04  & 1.43$\pm$0.02   & 4.05$\pm$0.02   & 0.55$\pm$0.01         & 1.41$\pm$0.04  \\
\multicolumn{1}{l|}{}                            & \multicolumn{1}{l|}{$a=1.1$}  & 1.25$\pm$0.01 & 3.62$\pm$0.04   & 0.49$\pm$0.00         & 6.28$\pm$0.48  & 1.26$\pm$0.02   & 3.61$\pm$0.03   & 0.49$\pm$0.01         & 6.61$\pm$0.73  \\ \midrule
\multicolumn{2}{c|}{Dataset}                                                     & \multicolumn{8}{c}{Adult}                                                                                                                             \\ \midrule
\multicolumn{2}{c|}{Termination Threshold}                                       & \multicolumn{4}{c}{$\epsilon=1e-4$}                                      & \multicolumn{4}{c}{{\ul $\epsilon=5e-4$}}                                  \\
\multicolumn{2}{c|}{Bargaining Cost}                                             & Net Profit    & Payment         & Realized $\DG$ (1e-2) & C(T)           & Net Profit      & Payment         & Realized $\DG$ (1e-2) & C(T)           \\ \midrule
\multicolumn{2}{c|}{No cost}                                                     & 1.80$\pm$0.02 & 0.61$\pm$0.02   & 3.01$\pm$0.01         & -              & 1.77$\pm$0.01   & 0.62$\pm$0.01   & 2.98$\pm$0.01         & -              \\
\multicolumn{1}{l|}{\multirow{2}{*}{$C(T)=aT$}}  & \multicolumn{1}{l|}{$a=0.1$}  & 1.70$\pm$0.02 & 0.61$\pm$0.02   & 2.90$\pm$0.01         & 3.87$\pm$0.25  & 1.68$\pm$0.01   & 0.63$\pm$0.01   & 2.89$\pm$0.00         & 4.02$\pm$0.26  \\
\multicolumn{1}{l|}{}                            & \multicolumn{1}{l|}{$a=1$}    & 1.16$\pm$0.00 & 0.52$\pm$0.00   & 2.10$\pm$0.00         & 1.00$\pm$0.00  & 1.17$\pm$0.00   & 0.51$\pm$0.00   & 2.10$\pm$0.00         & 1.00$\pm$0.00  \\ \cmidrule(r){1-2}
\multicolumn{1}{l|}{\multirow{2}{*}{$C(T)=a^T$}} & \multicolumn{1}{l|}{$a=1.01$} & 1.80$\pm$0.02 & 0.61$\pm$0.02   & 3.01$\pm$0.00         & 1.62$\pm$0.06  & 1.78$\pm$0.01   & 0.61$\pm$0.01   & 2.98$\pm$0.01         & 1.58$\pm$0.06  \\
\multicolumn{1}{l|}{}                            & $a=1.1$                       & 1.36$\pm$0.02 & 0.55$\pm$0.01   & 2.39$\pm$0.02         & 4.77$\pm$0.24  & 1.35$\pm$0.02   & 0.56$\pm$0.01   & 2.38$\pm$0.03         & 4.79$\pm$0.26  \\ \bottomrule
\end{tabular}
\label{tb:cost}
\end{table*}

The left three columns of the figures show that the proposed model outperforms the non-strategic variants in terms of net profit received by the task party and realized $\DG$, while the reward received by the task party is comparable or reasonable lower than the Increase Price, indicating that the model is more effective in achieving a win-win outcome for both parties. It is worth mentioning that the two parties reach agreements faster under the proposed approach compared with the Increase Price method, with a more robust trend, which suggests that targeting the turning point of the payment and net profit function is an effective strategy for reaching equilibrium. Moreover, the Random Bundle method performs poorly in almost all cases, with early terminations that result from unsuccessful transactions. This is due to the fact that feature bundles with low performance gain can be offered by the task party, in which case the task party terminates the bargaining for the violation of Case 4 in Section \ref{sec:per_termi}. 

From the right two columns of the figures, we can observe that the proposed model achieves a better alignment between the final quoted prices and the reserved price of the data party, indicating that the proposed model is more effective in achieving an outcome that is acceptable by both parties. While for Increase Price, as there is no constraint on the price offer, there are cases that the task party over-pays a feature bundle, \ie higher than the reserved price of the data party.

\subsection{Effect of Bargaining Cost}
In this section, we study how the bargaining cost affects the results of the proposed approach. We are interested in two situations: when the cost is linear to the bargaining rounds and when it is exponential to the bargaining rounds, \ie $C(T)=aT$ and 2) $C(T)=a^T$. We vary the cost factor $a$ and termination condition related hyper-parameter $\epsilon$ in Section \ref{sec:per_termi} to further analyze the effect of bargaining cost. For simplicity, we set $10*C_t(T)=10*C_d(T)=C(T)$ for the Credit and Adult dataset. We use $\epsilon_d=\epsilon_t=\epsilon$ for both datasets. We employ different $\epsilon$ on different datasets, the one underlined is set as the default value. We run the bargaining 100 times and report the mean and standard deviation of the net profit, payment, and cost of the final bargaining state. Note that the net profit and payment are calculated as the revenue before minus cost. The results under Random Forest base model are illustrated in Table \ref{tb:cost}, and those of the original perfect performance information setting are also included for comparison. The initial states of the current evaluation are the same for all runs. 

Table \ref{tb:cost} shows that when cost is introduced to the bargaining model, the net profit, payment, and realized $\DG$ are all generally smaller compared to the No Cost version, which indicates the two parties tend to reach a less optimal equilibrium. Such a trend is more significant when the cost increases faster (\ie larger $a$), which makes long-term negotiation less tolerable. $\epsilon$ also plays an important role in determining the final state of bargaining. It can be observed that a smaller $\epsilon$ generally leads to higher revenue for both parties, as it indicates the final realized $\DG$ is closer to the target one. However, this also leads to higher bargaining costs, which might offset the advantage of increased revenue. Therefore, as a player, it is important to consider the trade-off between bargaining efficiency and better equilibrium during decision making in the feature bargaining.

\begin{table}[]
\centering
\caption{Bargaining results under imperfect performance information.}
\fontsize{6.5}{7.5}\selectfont
\setlength{\tabcolsep}{0.6mm}{
\begin{tabular}{l|llllll}
\specialrule{0.3mm}{0em}{0em} 
Base Model                                           & \multicolumn{6}{c}{Random Forest}                                                                                                                                                       \\ \hline
Dataset                                              & \multicolumn{2}{c}{Titanic}                                 & \multicolumn{2}{c}{Credit}                                  & \multicolumn{2}{c}{Adult}                                   \\ \hline
Setting                                              & \multicolumn{1}{c}{Imperfect} & \multicolumn{1}{c}{Perfect} & \multicolumn{1}{c}{Imperfect} & \multicolumn{1}{c}{Perfect} & \multicolumn{1}{c}{Imperfect} & \multicolumn{1}{c}{Perfect} \\ \hline
$p$                                                  & 9.62$\pm$3.68                 & 9.22$\pm$0.26               & 12.90$\pm$0.62                & 9.43$\pm$0.28               & 12.53$\pm$1.27                & 10.14$\pm$0.34              \\
$P_0$                                                & 1.28$\pm$0.55                 & 1.33$\pm$0.01               & 1.41$\pm$0.95                 & 1.37$\pm$0.03               & 1.29$\pm$0.70                 & 1.46$\pm$0.01               \\
$P_h$                                                & 0.76$\pm$3.68                 & 0.37$\pm$0.26               & 3.88$\pm$0.62                 & 0.43$\pm$0.28               & 2.70$\pm$1.27                 & 0.32$\pm$0.34               \\
$\Delta p$                                           & -0.04$\pm$0.55                & 0.01$\pm$0.01               & 0.06$\pm$0.09                 & 0.03$\pm$0.03               & -0.19$\pm$0.70                & 0.00$\pm$0.01               \\
$\Delta P_0$                                         & 2.94$\pm$0.40                 & 2.92$\pm$0.05               & 1.49$\pm$0.95                 & 1.42$\pm$0.03               & 1.67$\pm$0.69                 & 1.77$\pm$0.01               \\
$\Delta G$                                           & 0.12$\pm$0.05                 & 0.17$\pm$0.01               & 0.01$\pm$0.01                 & 0.01$\pm$0.01               & 0.01$\pm$0.01                 & 0.03$\pm$0.00               \\
\begin{tabular}[c]{@{}l@{}}Net  Profit\end{tabular} & 117.81$\pm$46.70              & 170.01$\pm$0.05             & 2.03$\pm$0.05                 & 4.06$\pm$0.03               & 0.33$\pm$0.21                 & 0.63$\pm$0.02               \\
Payment                                              & 2.29$\pm$0.97                 & 2.92$\pm$0.05               & 0.64$\pm$0.03                 & 1.42$\pm$0.03               & 0.55$\pm$1.00                 & 1.76$\pm$0.01               \\ \hline
Base Model                                           & \multicolumn{6}{c}{3-layer MLP}                                                                                                                                                         \\ \hline
Dataset                                              & \multicolumn{2}{c}{Titanic}                                 & \multicolumn{2}{c}{Credit}                                  & \multicolumn{2}{c}{Adult}                                   \\ \hline
Setting                                              & \multicolumn{1}{c}{Imperfect} & \multicolumn{1}{c}{Perfect} & \multicolumn{1}{c}{Imperfect} & \multicolumn{1}{c}{Perfect} & \multicolumn{1}{c}{Imperfect} & \multicolumn{1}{c}{Perfect} \\ \hline
$p$                                                  & 11.24$\pm$3.14                & 9.85$\pm$0.31               & 12.88$\pm$2.10                & 10.30$\pm$0.44              & 12.41$\pm$1.95                & 10.44$\pm$0.37              \\
$P_0$                                                & 1.24$\pm$0.37                 & 1.41$\pm$0.01               & 1.50$\pm$0.76                 & 1.45$\pm$0.01               & 1.30$\pm$0.69                 & 1.47$\pm$0.01               \\
$P_h$                                                & 1.91$\pm$0.41                 & 0.51$\pm$0.31               & 2.35$\pm$1.10                 & 0.76$\pm$0.44               & 1.49$\pm$0.95                 & 0.53$\pm$0.37               \\
$\Delta p$                                           & -0.46$\pm$0.77                & 0.01$\pm$0.01               & 0.06$\pm$0.76                 & 0.02$\pm$0.01               & -0.19$\pm$0.19                & -0.03$\pm$0.01              \\
$\Delta P_0$                                         & 3.56$\pm$0.55                 & 3.51$\pm$0.07               & 1.71$\pm$0.76                 & 1.62$\pm$0.01               & 1.88$\pm$0.65                 & 1.92$\pm$0.02               \\
$\Delta G$                                           & 0.12$\pm$0.09                 & 0.21$\pm$0.00               & 0.01$\pm$0.00                 & 0.02$\pm$0.00               & 0.02$\pm$0.01                 & 0.04$\pm$0.00               \\
\begin{tabular}[c]{@{}l@{}}Net Profit\end{tabular} & 144.46$\pm$94.42              & 210.23$\pm$0.07             & 7.41$\pm$0.09                 & 14.72$\pm$0.01              & 1.36$\pm$0.12                 & 1.50$\pm$0.02               \\
Payment                                              & 2.04$\pm$1.72                 & 3.51$\pm$0.07               & 1.27$\pm$0.03                 & 1.62$\pm$0.01               & 1.26$\pm$0.07                 & 1.91$\pm$0.02               \\ \specialrule{0.3mm}{0em}{0em} 
\end{tabular}}
\label{tb:imp}
\end{table}

\subsection{Imperfect Performance Information Bargaining Analysis}

In this section, we conduct experiments under imperfect performance information setting. For the task party, a 3-layer MLP with embedding dimensions 64, 32, 16 are used for the estimation of $\DG$. For the data party, we first embed each singular feature with the nn.Embedding layer \footnote{\url{https://pytorch.org/docs/stable/generated/torch.nn.Embedding.html}} in PyTorch and then take the average of each feature variable's embedding as the representation of the whole feature bundle, which is then used as input of the performance gain estimation network. The estimation network on the task party follows the same structure as the task party. We set the initial 100 rounds as the exploration rounds of bargaining, \ie $N=100$, during which the bargaining will not fail under any cases.

\begin{figure} [t]
\subfigure[Random Forest (Titanic)]{
\centering
\includegraphics[width=0.46\linewidth]{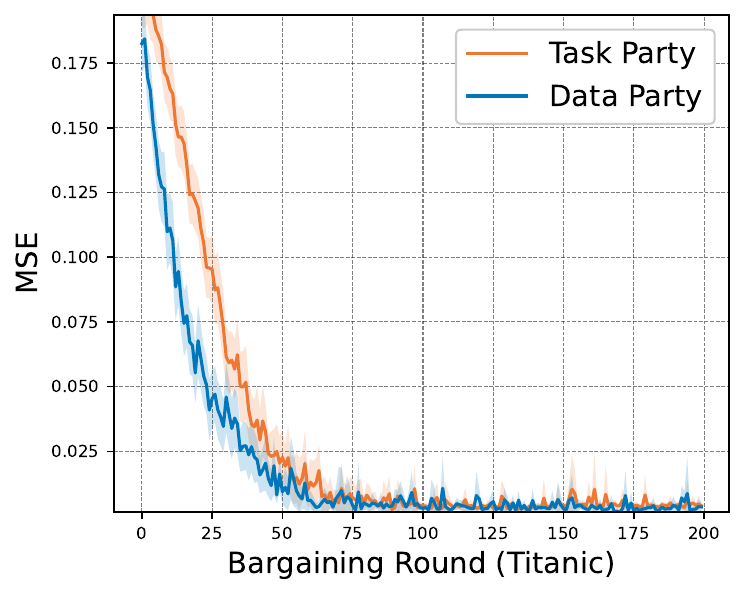}
}%
\subfigure[3-Layer MLP (Titanic) ]{
\centering
\includegraphics[width=0.46\linewidth]{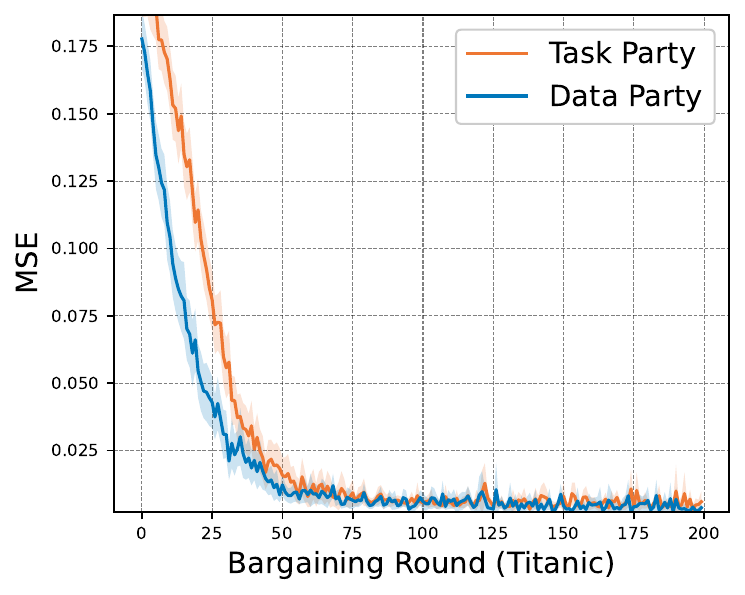}
}%

\subfigure[Random Forest (Credit)]{
\centering
\includegraphics[width=0.46\linewidth]{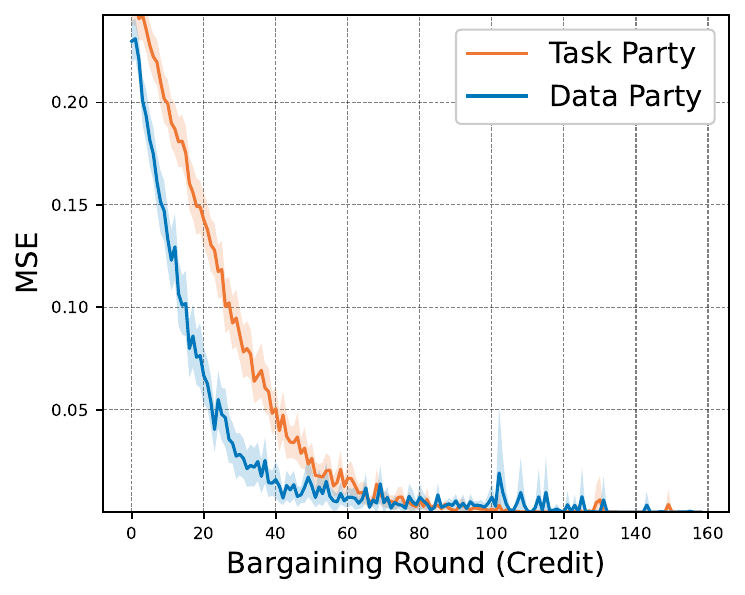}
}%
\subfigure[3-Layer MLP (Credit) ]{
\centering
\includegraphics[width=0.46\linewidth]{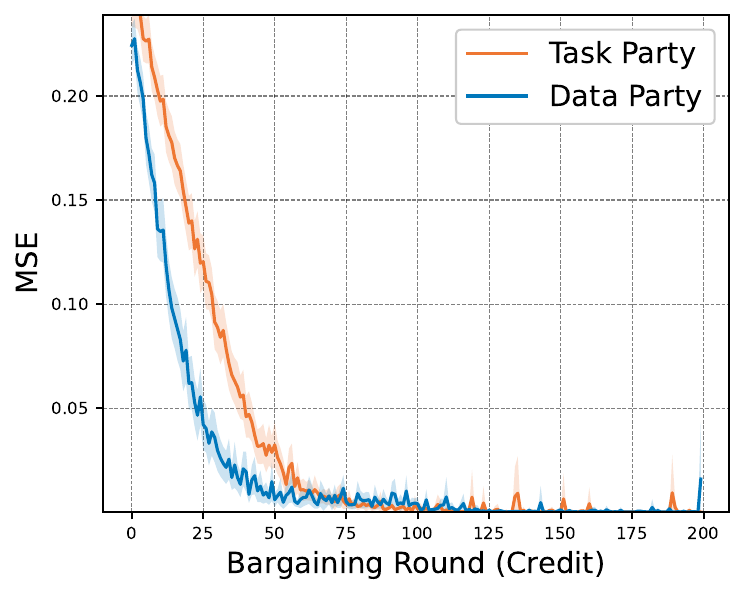}
}%

\subfigure[Random Forest (Adult)]{
\centering
\includegraphics[width=0.46\linewidth]{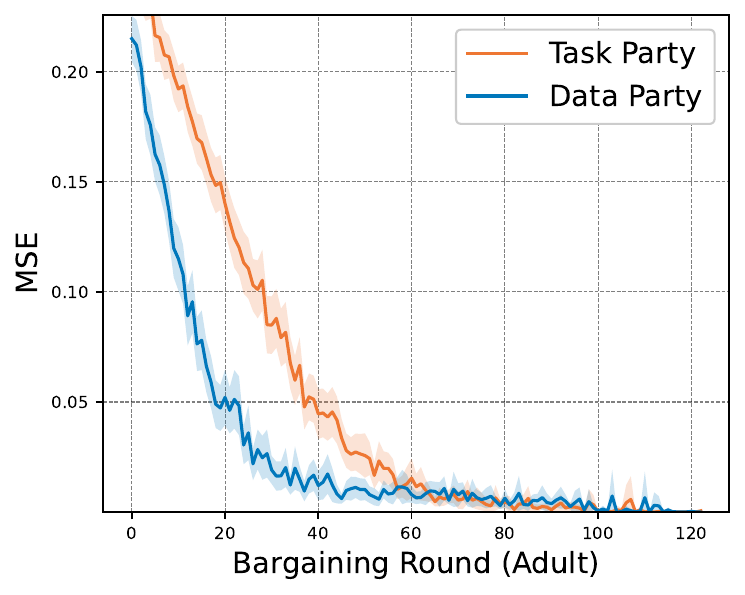}
}%
\subfigure[3-Layer MLP (Adult) ]{
\centering
\includegraphics[width=0.46\linewidth]{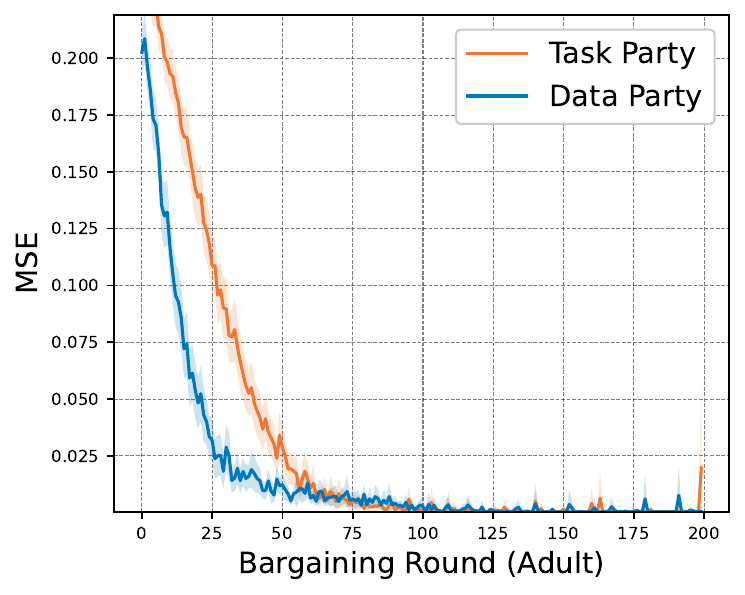}
}%
\centering
\caption{The MSE of $\DG$ estimation networks on two parties.}
\label{fig:convergence}
\end{figure}

\subsubsection{Comparison with Bargaining with Perfect Performance Information}
To evaluate if the bargaining under imperfect performance information is effective, we compare the bargaining-related final and intermediate variables with the perfect performance information setting and obtain the results in Table \ref{tb:imp}. $\Delta p$ and $\Delta P_0$ denote the difference value of $p-p_l$ and $P_0-P_{l}$, where $p_l$ and $P_l$ are the reserved prices of the task party's target feature bundle. For the bargaining termination conditions, we set $\epsilon_d=\epsilon_t=5e-2$ on the Titanic dataset, $1e-3$ on the Credit dataset and $5e-3$ on the Adult dataset. Due to limited space, we only report the results on Credit and Adult dataset. The initial state of the bargaining is set as the same for the two settings. We run the bargaining 100 times and report the mean and std values. We record that the corresponding values are negative infinitely small if the bargaining terminates with the transaction fails. 

We can conclude from the table that the bargaining model under imperfect performance information is effective and comparable to the perfect performance information setting in most cases, as the final $\DG$, net profit and payment are usually of reasonable magnitudes. Moreover, the std values demonstrate that the variability in the results is also relatively low with no infinitely large values, indicating that the bargaining process is consistent and stable across multiple runs, even under imperfect performance information. This highlights the robustness of the bargaining model and its potential applicability in real-world scenarios where perfect information may not always be available.

\subsubsection{$\DG$ Estimation Convergence}

It can be found in Figure \ref{fig:convergence} that the convergence of the estimation networks on both parties over the course of the bargaining rounds. The estimation networks converge quickly within the first 20-30 rounds, and the accuracy of the estimation improves gradually with more rounds of bargaining. At around 100, the estimation is precious enough to conduct the true bargaining, which verifies the effectiveness of the estimation while bargaining approach.

\section{Related Work}

\textbf{Federated Learning}. FL approaches ensure data privacy as raw data never leaves its original location. According how data is partitioned in the sample and feature space , FL can be broadly classified into three main categories: horizontal federated learning (HFL) \cite{li2022federated,shi2023towards,zhu2021federated},  vertical federated learning (VFL) \cite{liu2022vertical,wei2022vertical,li2023vertical}, and Federated Transfer Learning (FTL) \cite{yang2019federated}. HFL allows collaborative training on clients with different data samples but the same features space. VFL enables training of clients' datasets sharing the same samples while holding different features. FTL deals with the scenario where clients' datasets differ in both sample and feature space. Besides importing FL models' effectiveness and efficiency, the commercialization of FL has attracted increasing attention. Most current works focus on evaluating the contribution of participants \cite{song2019profit,wang2020principled,liu2022gtg} or proposing incentive mechanism to facilitate participation of data owners \cite{tu2022incentive,zhan2021survey,sim2020collaborative}, while the model/data trading and pricing in FL has only limited studied \cite{zheng2022fl}.\\
\textbf{Data Pricing and Valuation}. There is a growing amount of literature on data valuation, data trading, and pricing mechanisms in data markets, particularly in the context of digital products and services \cite{pei2020survey,driessen2022data,azcoitia2022survey}. Researchers have proposed various approaches to determine the value of data, such as market-based approaches, cost-based approaches, and income-based approaches. Market-based approaches rely on supply and demand dynamics to determine the value of data \cite{chawla2019revenue,fernandez2020data,muschalle2013pricing}, while cost-based approaches consider the cost of acquiring, storing, and processing data \cite{zheng2022fl,acquisti2016economics,fleischer2012approximately}. Income-based approaches, on the other hand, look at the potential revenue generated from the use of data/models \cite{cong2022data}. Our bargaining model falls into income-based approach. In terms of pricing mechanisms, subscription-based \cite{alaei2023optimal}, and auction-based \cite{jiao2017profit,cao2017data,gao2018privacy,zhang2021optimizing} pricing models have been proposed and used in data markets. 
In the VFL market, we choose to use one-one bargaining-based solution for that VFL is usually conducted between one task party and one data party and the iterative process of bargaining allows for the probability of reaching mutual-beneficial outcomes.
\section{Conclusions and Future Work}
In this paper, we have identified the need for an economically efficient approach to feature trading in VFL and proposed a bargaining-based model to address this issue. Our model incorporates performance gain-based pricing and analyzes the bargaining process under perfect and imperfect performance information settings. We have demonstrated the existence of an equilibrium that optimizes the objectives of both the task party and the data party in perfect performance information setting. Additionally, we extend the result to imperfect performance information scenarios and propose performance gain estimation-based bargaining strategies. We also discuss potential security issues and solutions. Experiments on real-world datasets have verified the effectiveness of the proposed bargaining model. 

However, the proposed bargaining model has several limitations. 1) It does not provide protection if the participants manipulate the goods or information when terminate the game. For example, the task party may accept a feature bundle with high performance gain but only report a lower value to reduce its payment. A possible solution for this is to involve of a trustworthy third party for evaluation. 2) The sampling-evaluation based quoted pricing choosing strategy is straightforward but not efficient and the task party can employ automatic bargaining offer strategy, such as learning based, to optimize the efficient of offer generating. Nevertheless, we maintain a positive view of the impact of this paper and hope it serves as a foundational work on FL markets and inspires future research.

\bibliographystyle{IEEEtran}
\bibliography{sample}

\end{document}